%% file: iclr2027_conference.tex
\documentclass{article}
\usepackage{iclr2027_conference,times}

\input{math_commands.tex}

\usepackage{amsmath}
\usepackage{amssymb}
\usepackage{hyperref}
\usepackage{url}
\usepackage{algorithm}
\usepackage{algpseudocode}

\usepackage{amsthm}
\makeatletter
\g@addto@macro\th@plain{\thm@notefont{\bfseries\upshape}}
\makeatother

\usepackage{graphicx}
\usepackage{booktabs}
\usepackage{multirow}
\usepackage[table,dvipsnames]{xcolor}

\usepackage{subcaption}
\newcommand{\basex}[1]{\textcolor{gray}{(#1)}}
\newcommand{\uprise}[1]{\textcolor{OliveGreen}{$\uparrow$#1}}
\newcommand{\down}[1]{\textcolor{Maroon}{$\downarrow$#1}}

\renewcommand{\uprise}[1]{%
  \makebox[2.5em][l]{\textcolor{OliveGreen}{(\ensuremath{\uparrow}#1)}}%
}

\renewcommand{\down}[1]{%
  \makebox[2.5em][l]{\textcolor{Maroon}{(\ensuremath{\downarrow}#1)}}%
}

\renewcommand{\basex}[1]{%
  \makebox[2.5em][l]{\textcolor{gray}{(#1)}}%
}

\newcommand{\cocarestatedname}{}
\newtheorem*{cocarestatedinner}{\cocarestatedname}
\newenvironment{cocarestated}[2]{%
  \renewcommand{\cocarestatedname}{Proposition~\ref{#1}}%
  \begin{cocarestatedinner}[#2]%
}{\end{cocarestatedinner}}

\title{CoCA: Criterion-Aware Credit Assignment \\for Rubric-Based Reinforcement Learning}

\author{
\begin{tabular}{@{}l@{}}
\mbox{\textbf{Guangcheng Zhu}$^{12}$%
\thanks{Equal Contribution. $\dagger$ Corresponding Authors.}%
\hphantom{$^{*}$}}\hspace{0.7em}%
\mbox{\textbf{Yingfan Ma}$^{2\mkern0.5mu*}$}\hspace{0.7em}%
\mbox{\textbf{Xing Zheng}$^{2}$}\hspace{0.7em}%
\mbox{\textbf{Haobo Wang}$^{1\dagger}$}\hspace{0.7em}%
\mbox{\textbf{Shenzhi Yang}$^{12}$}\\
\mbox{\textbf{Kai Tang}$^{12}$}\hspace{0.7em}%
\mbox{\textbf{Zhengqing Zang}$^{12}$}\hspace{0.7em}%
\mbox{\textbf{Bowen Song}$^{2\dagger}$}\hspace{0.7em}%
\mbox{\textbf{Weiqiang Wang}$^{2}$}\hspace{0.7em}%
\mbox{\textbf{Gang Chen}$^{1}$}
\end{tabular}
\vspace{0mm} \\
\textsuperscript{1} Zhejiang University \quad
\textsuperscript{2} Ant Group \\
\texttt{\small\{zhuguangcheng,wanghaobo,yangshenzhi,tk0819,zangzq,cg\}@zju.edu.cn}\\
\small\texttt{\{mayingfan.myf,feishang.zx,bowen.sbw,weiqiang.wwq\}@antgroup.com}
}

\newtheorem{proposition}{Proposition}

\newtheorem{remark}{Remark}
\newcommand{\method}{CoCA}

\iclrfinalcopy
\begin{document}

\maketitle
\lhead{Preprint}

\begin{abstract}
Rubric-based reinforcement learning (RL) evaluates responses against explicit criteria to construct rewards, offering a viable way to train large language models on open-ended tasks.
However, standard methods compress criterion-level feedback into scalar rewards and assign uniform token advantages, leaving structured rubric supervision underutilized.
Recent work seeks finer-grained guidance by treating rubrics as privileged information to derive self-distillation targets, but the resulting learning signal is not grounded in criterion outcomes and can lead to reward degradation.
This motivates us to reframe the problem as criterion-aware credit assignment, where privileged rubric information serves as a relevance cue for credit allocation rather than a source of imitation targets.
Accordingly, we propose \textbf{CoCA}, which derives criterion-level credits from group-relative outcomes and allocates them across tokens based on criterion-induced distribution shifts, yielding both outcome-grounded and fine-grained learning signals.
Empirically, CoCA achieves an average gain of \textbf{+4.2\%} over standard rubric-based RL across medicine and science benchmarks, demonstrating more effective learning on open-ended tasks.
Our code is available at \href{https://github.com/gczhu/CoCA}{https://github.com/gczhu/CoCA}.
\end{abstract}

\section{Introduction}
\label{sec:introduction}

Scaling reinforcement learning (RL) to open-ended tasks remains a central challenge in the post-training of large language models (LLMs)~\citep{rlrmsurvey,writingzero,rewardchain}.
Unlike verifiable tasks, which often permit deterministic evaluation, open-ended tasks such as medical consultation~\citep{healthbench} and scientific explanation~\citep{sciarena} admit diverse valid responses and require multidimensional quality assessment.
This makes it difficult to construct reliable and scalable reward signals.
Rubrics provide a promising approach by decomposing overall response quality into explicit, fine-grained criteria that can be assessed individually~\citep{llmrubric,checklists,rubricsurvey}.
By specifying the requirements of a high-quality response along multiple dimensions, they offer a structured and interpretable foundation for scalable supervision in open-ended tasks~\citep{rubricanchors,dece,openrubrics}.

To translate this supervision into policy updates, standard rubric-based RL~\citep{rar,drtulu,rgr} uses an LLM judge~\citep{llmjudge} to evaluate each rollout against the rubric criteria, aggregates the criterion-level outcomes into a scalar reward, and optimizes the policy with Group Relative Policy Optimization (GRPO)~\citep{grpo}.
Although this pipeline makes rubric feedback compatible with standard policy optimization, it reduces the multidimensional assessment to a single response-level learning signal.
As illustrated in the left panel of Figure~\ref{fig:motivation}, scalar aggregation can assign the same reward to responses with different criterion profiles, obscuring their distinct strengths and weaknesses.
Moreover, applying the same response-level advantage to every token may reinforce criterion-violating content when the advantage is positive and suppress criterion-satisfying content when it is negative.
Consequently, this pipeline performs fine-grained evaluation but learns from coarse response-level supervision, limiting how effectively the policy can exploit the structured feedback provided by rubrics.

This mismatch has prompted recent attempts~\citep{rgsd,rcsd,rupi} to derive token-level guidance through on-policy self-distillation (OPSD)~\citep{opsd,sdpo}.
Specifically, these methods train the policy on its own rollouts to match the next-token distribution of a self-teacher that receives rubrics as privileged context.
This distillation loss is used either as a standalone objective or as an auxiliary objective in rubric-based RL, with the expectation that such fine-grained guidance will help the policy learn more effectively from rubrics.
However, as shown in the right panel of Figure~\ref{fig:motivation}, reward declines when OPSD is optimized alone and adding it as an auxiliary loss degrades the performance of rubric-based RL.
Our theoretical analysis highlights a gap between better generation and reliable guidance: although a rubric-conditioned teacher can perform better overall, matching its predictions need not improve criterion satisfaction.
OPSD determines what to reinforce or suppress from teacher--student prediction differences rather than observed criterion outcomes, so closer imitation can come at the cost of lower reward.
This motivates us to reframe the problem as \textbf{criterion-aware credit assignment}.


Our key insight is to separate credit determination from allocation: observed outcomes determine the sign and budget of each criterion's credit, while privileged rubric information serves as a \textit{relevance cue} for allocating credit across tokens.
Based on this insight, we propose \textbf{CoCA}, a criterion-aware credit assignment framework for rubric-based RL.
Specifically, we first use group-relative criterion outcomes to decompose the response-level advantage into signed criterion-level credits, a step we call \textit{credit grounding}.
We then compute the magnitude of criterion-induced next-token distribution shifts as a proxy for token relevance.
Rather than matching the rubric-conditioned distributions as in OPSD, we use this relevance to redistribute each criterion's credit across tokens, a step we call \textit{credit resolution}.
Finally, we aggregate the resolved credits into token-level advantages, yielding outcome-grounded and fine-grained policy updates.
Our experiments in medicine and science show that CoCA outperforms all baselines, improving over standard rubric-based RL by \textbf{+4.2\%} on average.
Criterion-level analyses further show that CoCA corrects unmet criteria more effectively while better preserving those already satisfied.
Together, these results demonstrate that CoCA makes more effective use of structured rubric feedback and enables more targeted learning.

\begin{figure*}[!t]
    \centering
    \includegraphics[width=\textwidth]{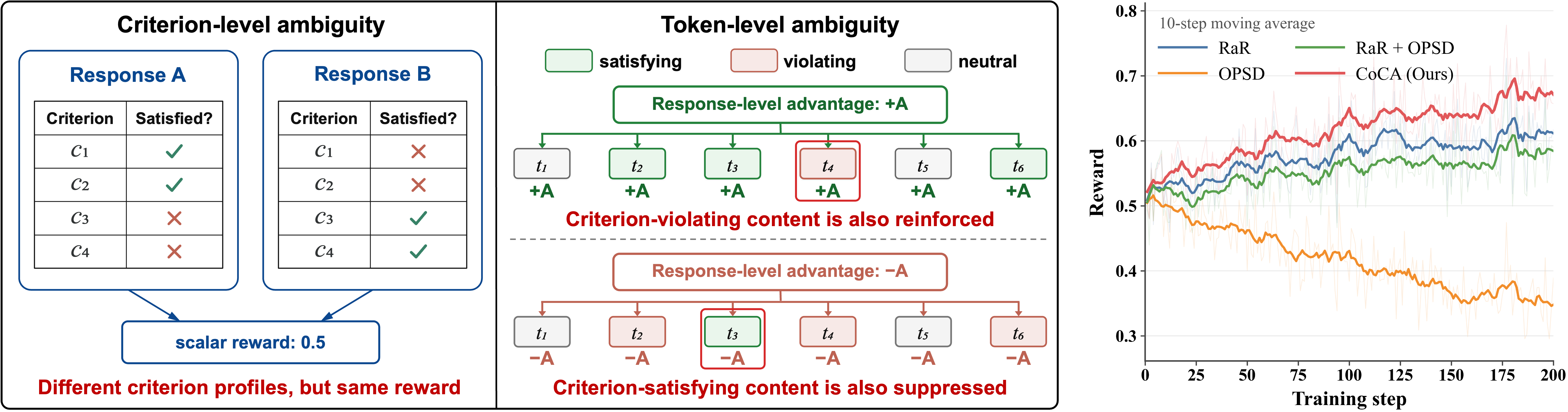}
    \caption{\small (Left) The mismatch between criterion-level evaluation and response-level learning signals in standard rubric-based RL. (Right) Training dynamics across methods on RaR-Medicine over the first 200 steps.}
    \vspace{-6pt}
    \label{fig:motivation}
\end{figure*}

\section{Preliminaries}
\label{sec:preliminaries}

\paragraph{Rubric-based RL.}
Rubrics specify response quality through a set of weighted, independently assessable criteria.
Formally, a training example $(x,\mathcal{R})\sim\mathcal{D}$ consists of a question $x$ and its associated rubric $\mathcal{R}=\{(c_j,s_j)\}_{j=1}^{M}$, where $c_j$ denotes the $j$-th evaluation criterion, $s_j$ its point value, and $M$ the number of criteria.
Following prior work~\citep{healthbench,rar}, we adopt a binary checklist-style scoring protocol and use GRPO~\citep{grpo} as the base RL algorithm.
For each question $x$, the rollout policy $\pi_{\theta_{\text{old}}}$ samples a group of $G$ responses $\{y_i\}_{i=1}^{G}$.
An LLM judge evaluates each response $y_i$ against each criterion $c_j$, yielding $z_{ij}\!\in\!\{0,1\}$, with $z_{ij}\!=\!1$ if $y_i$ meets $c_j$.
These criterion-level judgments are aggregated into a scalar reward $R_i=\sum_j s_jz_{ij}/\sum_{j:s_j>0} s_j$.
To avoid assigning disproportionate weight to rollout groups with low reward variance, we follow Dr. GRPO~\citep{drgrpo} and omit normalization by the within-group reward standard deviation.
We thus define the group-relative advantage as \(A_i = R_i - \bar R\), where \(\bar R = \frac{1}{G}\sum_{k=1}^{G} R_k\) is the mean reward within the group.
The policy $\pi_\theta$ is then optimized using the GRPO loss:
\begin{equation}
\mathcal{L}_{\mathrm{GRPO}} =
-\mathbb{E}_{x\sim\mathcal{D},\{y_i\}_{i=1}^G\sim\pi_{\theta_{\text{old}}}\!(\cdot|x)}
\!\left[\textstyle
\frac{1}{G}\!\sum_{i=1}^{G}\!
\frac{1}{|y_i|}\!\sum_{t=1}^{|y_i|}\!
\min\bigl(
\rho_{i,t}A_i,
\operatorname{clip}(\rho_{i,t},1\!-\!\epsilon,1\!+\!\epsilon)A_i
\bigr)\right],
\label{eq:grpo}
\end{equation}
where $\rho_{i,t}=\pi_\theta(y_{i,t}\!\mid\! x,y_{i,\!<t})/
\pi_{\theta_{\text{old}}}(y_{i,t}\!\mid\! x,y_{i,\!<t})$ is the importance ratio and $\epsilon$ is the clipping factor.
Hence, criterion-level evaluations are reduced to a response-level advantage shared by all tokens, without explicitly using each criterion's outcome to guide targeted updates to the relevant tokens.

\paragraph{Rubric-conditioned OPSD.}
Beyond their role in reward computation, rubrics can serve as privileged information $\xi$ for token-level supervision through OPSD~\citep{opsd,sdpo}.
Formally, for an on-policy rollout $y_i$, let $h_{i,t}=(x,y_{i,<t})$ and $h_{i,t}(\xi)=(x,\xi,y_{i,<t})$ denote the ordinary and privileged contexts at position $t$, respectively.
OPSD trains the policy by matching its next-token distribution under $h_{i,t}$ to that under $h_{i,t}(\xi)$.
We consider the reverse-KL objective:
\begin{equation}
\mathcal{L}_{\mathrm{OPSD}}
=
\mathbb{E}_{x\sim\mathcal{D},\,y_i\sim\pi_\theta(\cdot\mid x)}
\left[
\tfrac{1}{|y_i|}
\textstyle\sum_{t=1}^{|y_i|}
D_{\mathrm{KL}}
\left(
\pi_\theta(\cdot\mid h_{i,t})
\,\middle\|\,
\operatorname{sg}\!\left[
\pi_\theta(\cdot\mid h_{i,t}(\xi))
\right]
\right)
\right].
\label{eq:opsd}
\end{equation}
Here, $\operatorname{sg}[\cdot]$ denotes the stop-gradient operator, and the distribution under $h_{i,t}(\xi)$ serves as a detached self-teacher.
In the rubric-conditioned setting, $\xi$ is instantiated as either the full criterion set $\{c_j\}_{j=1}^{M}$ or a selected subset.
The resulting self-teacher translates these quality requirements into target next-token distributions for policy optimization.


\section{Credit Assignment in Rubric-Based Policy Optimization}
\label{sec:credit-assignment}
Rubric evaluation yields criterion-level outcomes that typically depend on different token decisions within a response.
Accordingly, translating this structured feedback into targeted policy updates involves two complementary components:
\emph{credit grounding} determines a signed credit budget for each criterion from observed outcomes, whereas \emph{credit resolution} allocates these budgets across relevant tokens while preserving their values.
Formally, for response $y_i$, let $\mathbf z_i=(z_{i1},\ldots,z_{iM})^\top$.
Credit grounding uses a rule $g$ to map these outcomes to a vector of signed criterion credits:
\begin{equation}
    \mathbf a_i
    =
    g(\mathbf z_i;\mathcal R,\mathcal B)
    =
    (A_{i1},\ldots,A_{iM})^\top,
    \label{eq:credit-grounding}
\end{equation}
where $A_{ij}$ is the credit budget for criterion $c_j$, and $\mathcal B$ denotes optional reference statistics derived from evaluation outcomes.
Given these budgets, credit resolution determines their allocation across token positions.
We represent this allocation by a criterion--token credit matrix $\mathbf C_i\in\mathbb R^{M\times T_i}$, where $T_i\!=\!|y_i|$ is the response length and $[\mathbf C_i]_{j,t}$ denotes the credit assigned to token $t$ for criterion $c_j$.
We consider allocations that preserve each criterion budget as its mean credit over tokens:
\begin{equation}
    \mathcal C_i(\mathbf a_i)
    =
    \left\{
        \mathbf C\in\mathbb R^{M\times T_i}
        :
        \tfrac{1}{T_i}\mathbf C\mathbf 1_{T_i}
        =
        \mathbf a_i
    \right\},
    \qquad
    \boldsymbol\alpha_i(\mathbf C)
    =
    \mathbf C^\top\mathbf 1_M,
    \label{eq:credit-resolution-space}
\end{equation}
where $\mathcal C_i(\mathbf a_i)$ denotes the set of grounding-preserving resolutions and $\boldsymbol\alpha_i(\mathbf C)\in\mathbb R^{T_i}$ is the resulting token-level credit vector.
We use $\mathbf C_i^*\in\mathcal C_i(\mathbf a_i)$ to denote an oracle resolution of the criterion budgets.

\subsection{Rubric-Based RL: Outcome-Grounded but Under-Resolved}
\label{sec:rubric-rl-resolution}
We first interpret standard rubric-based RL within this framework.
Let $\bar z_j=\frac{1}{G}\sum_{k=1}^{G}z_{kj}$ denote the mean score for criterion $j$ across the rollout group.
By the linearity of group-relative centering, the response-level advantage $A_i$ can be decomposed exactly into criterion-level contributions, each determined by a criterion's weight and its score relative to the group mean:
\begin{equation}
    A_{ij}
    =\frac{s_j(z_{ij}-\bar z_j)}
    {\sum_{\ell:s_\ell>0}s_\ell},
    \qquad
    A_i=\sum_{j=1}^{M}A_{ij}.
    \label{eq:criterion-advantage}
\end{equation}
Here, the grounding rule $g$ uses $\{\bar z_j\}_{j=1}^{M}$ as the reference statistics $\mathcal B$.
Each budget thus reflects the response's weighted performance relative to the group on the corresponding criterion.
However, standard rubric-based RL uses the aggregate advantage
$A_i\!=\!\mathbf 1_M^\top\mathbf a_i$ and assigns it uniformly to every token.
This is equivalent to the allocation $\mathbf a_i\mathbf 1_{T_i}^{\top}\!\in\!\mathcal C_i(\mathbf a_i)$, which preserves the criterion budgets but does not account for which tokens are relevant to each criterion.
To formalize this loss of resolution, let $\mathcal U_i=\{\mathbf u\mathbf 1_{T_i}^{\top}:\mathbf u\in\mathbb R^M\}$ denote the subspace of criterion--token credit matrices that are uniform across tokens, and let $\Pi_{\mathcal U_i}$ denote the orthogonal projection onto this subspace.

\begin{proposition}[Credit resolution collapse in rubric-based RL]
\label{prop:rubric-rl-collapse}
For any oracle resolution $\mathbf C_i^*\in\mathcal C_i(\mathbf a_i)$, the token credit induced by rubric-based RL and the corresponding projection error satisfy
\begin{equation}
    \left(\mathbf 1_M^\top\mathbf a_i\right)\mathbf 1_{T_i}
    =
    A_i\mathbf 1_{T_i}
    =
    \boldsymbol\alpha_i
    \!\left(\Pi_{\mathcal U_i}(\mathbf C_i^*)\right),
    \quad
    \tfrac{1}{T_i}
    \left\|
        \mathbf C_i^*-\Pi_{\mathcal U_i}(\mathbf C_i^*)
    \right\|_F^2
    =
    {\textstyle\sum_{j=1}^{M}}
    \operatorname{Var}_{t}
    \left([\mathbf C_i^*]_{j,t}\right).
    \label{eq:rubric-rl-collapse}
\end{equation}
Moreover, the linear map $\mathbf a\mapsto(\mathbf 1_M^\top\mathbf a)\mathbf 1_{T_i}$ has rank one and kernel $\mathbf 1_M^\perp$.
Consequently, any nonzero criterion budget vector $\mathbf a_i\in\mathbf 1_M^\perp$ yields a zero token-credit vector, while any position $t$ satisfying $A_i[\boldsymbol\alpha_i(\mathbf C_i^*)]_t<0$ is assigned credit whose sign is opposite to its oracle token credit.
\end{proposition}

Proposition~\ref{prop:rubric-rl-collapse} shows that criterion aggregation retains only the sum $A_i$, obscuring the underlying pattern of strengths and weaknesses and allowing opposing criterion credits to cancel.
Uniform allocation further removes within-criterion variation across tokens, giving every token the same credit $A_i$ even when the oracle calls for reinforcing some tokens and suppressing others.
The resulting learning signal remains outcome-grounded but leaves criterion-specific token credit under-resolved.

\subsection{OPSD: Fine-Grained but Not Necessarily Outcome-Grounded}
\label{sec:opsd-grounding}

In contrast, OPSD provides token-specific guidance by matching the policy to a rubric-conditioned self-teacher.
Although this teacher may generate better responses on its own, its value for distillation depends on the matching updates it induces.
Yet, these updates are determined by teacher--student prediction differences without accounting for which rubric criteria the sampled response satisfies or violates.
Improving the match to the teacher need not improve criterion satisfaction and may instead reduce reward.
To examine this gap locally, let $r_{i,t}(a;\xi)=\log\pi_\theta(a\!\mid\! h_{i,t}(\xi))\!-\!\log\pi_\theta(a\!\mid\! h_{i,t})$ denote the teacher-induced matching signal at token context $h_{i,t}$, and define $J(\theta;h)=\mathbb E_{a\sim\pi_\theta(\cdot\mid h)}[R(a)]$.


\begin{proposition}[OPSD matching can reduce reward]
\label{prop:opsd-grounding}
Consider the next-token decision at $h_{i,t}$ as a one-step contextual bandit with finite action space $\mathcal V$, a reward function $R:\mathcal V\to\mathbb R$, and a tabular softmax policy.
Let $\theta'$ be obtained by a natural-gradient descent step of size $\eta>0$ on the KL term of Eq.~(\ref{eq:opsd}), using the local Fisher information and holding the self-teacher fixed.
Then, as $\eta\to0$,
\begin{equation}
    J(\theta';h_{i,t})-J(\theta;h_{i,t})
    =
    \eta\,
    \operatorname{Cov}_{a\sim\pi_\theta(\cdot\mid h_{i,t})}
    \left(
        R(a),
        r_{i,t}(a;\xi)
    \right)
    +
    \mathcal O(\eta^2).
    \label{eq:opsd-grounding}
\end{equation}
When $|\mathcal V|\geq3$, there exist instances for which $J(\theta;h_{i,t}(\xi))\!>\!J(\theta;h_{i,t})$ while the covariance in Eq.~(\ref{eq:opsd-grounding}) is strictly negative, implying $J(\theta';h_{i,t})\!<\!J(\theta;h_{i,t})$ for sufficiently small $\eta>0$.
\end{proposition}

Thus, the first-order reward change depends on how the matching signal aligns with reward under the student policy.
Intuitively, the local matching update favors actions with larger teacher--student log-probability differences. 
The teacher's higher expected reward reflects the quality of its overall action distribution, but does not guarantee that these relative preferences align with reward.
When this alignment is negative, a sufficiently small matching step reduces the local distillation loss while lowering expected reward.
This local failure mode provides a plausible explanation for the reward degradation observed with OPSD in Figure~\ref{fig:motivation}.


\begin{remark}
\label{rem:loss-combination}
Proposition~\ref{prop:opsd-grounding} also helps explain why introducing an auxiliary OPSD loss can undermine rubric-based RL.
The added teacher-matching signal lacks outcome grounding and may counteract reward-driven learning, while leaving the original coarse credit allocation unchanged.
We therefore seek to refine the allocation of outcome-grounded credit rather than add an independent imitation objective.
Accordingly, privileged rubric information should serve as a relevance cue rather than a source of imitation targets, guiding token-level allocation while observed outcomes determine each criterion's signed credit budget.
\end{remark}

\begin{figure*}[!t]
    \centering
    \includegraphics[width=\textwidth]{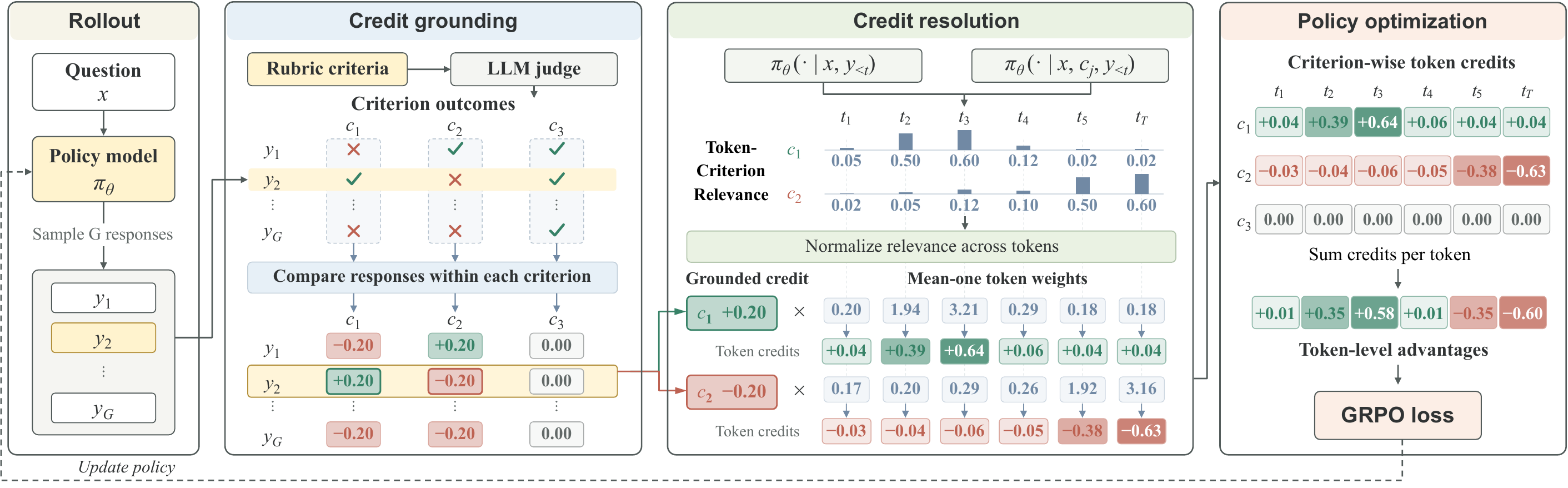}
    \caption{\small Overview of CoCA. Sampled responses are first evaluated against a shared rubric to derive signed criterion credits through within-group comparisons. Guided by relevance cues from criterion-conditioned distribution shifts, these credits are then allocated across tokens while preserving each criterion's signed budget. Finally, the resolved credits are aggregated into token-level advantages for GRPO optimization.}
    \label{fig:method}
    \vspace{-6pt}
\end{figure*}

\section{\method{}: Criterion-Aware Credit Assignment}
\label{sec:method}

Building on this principle, we propose \method{}, a criterion-aware credit assignment framework for rubric-based RL.
As illustrated in Figure~\ref{fig:method}, \method{} assigns credit in two stages.
\emph{Credit grounding} determines each criterion's signed contribution to a response via within-group comparisons of rubric outcomes, while \emph{credit resolution} allocates that contribution to tokens based on token--criterion relevance.
These stages provide criterion-aware token-level learning signals for GRPO optimization.

\subsection{Token-Level Advantage Construction}
\label{sec:credit-resolution}

\paragraph{Credit grounding.}
We instantiate the grounding rule $g$ with the group-relative decomposition in Eq.~(\ref{eq:criterion-advantage}).
For each response $y_i$, this yields signed criterion budgets $A_{ij}$ satisfying $\sum_j A_{ij}=A_i$.

\paragraph{Credit resolution.}
To resolve each budget $A_{ij}$, we use criterion $c_j$ as privileged information and compare the frozen rollout policy's next-token distributions under $h_{i,t}$ and $h_{i,t}(c_j)$.
We collect the relevance scores in $\mathbf d_{ij}=(d_{ij,1},\ldots,d_{ij,T_i})^\top$, where
\begin{equation}
    d_{ij,t}
    =
    \operatorname{sg}\!\left[
        D_{\mathrm{JS}}
        \left(
            \pi_{\theta_{\mathrm{old}}}(\cdot\mid h_{i,t})
            \,\middle\|\,
            \pi_{\theta_{\mathrm{old}}}(\cdot\mid h_{i,t}(c_j))
        \right)
    \right].
    \label{eq:criterion-relevance}
\end{equation}
Jensen--Shannon divergence provides a symmetric, bounded measure of how strongly criterion conditioning changes the next-token distribution.
We use this sensitivity as an unsigned relevance cue, while retaining $A_{ij}$ as the outcome-grounded credit budget.
We then transform these scores into positive token weights with mean one:
\begin{equation}
    \boldsymbol\omega_{ij}
    =
    T_i\operatorname{softmax}(\beta\mathbf d_{ij}),
    \qquad
    \tfrac{1}{T_i}
    \mathbf 1_{T_i}^{\top}\boldsymbol\omega_{ij}
    =
    1.
    \label{eq:criterion-token-weight}
\end{equation}
The parameter $\beta\geq0$ controls allocation concentration.
Multiplying $A_{ij}$ by these weights preserves its sign and ensures that the mean token credit equals $A_{ij}$.
For efficiency, we restrict nonuniform allocation to
$\mathcal S_i=\operatorname{TopK}_{j:A_{ij}\neq0}(|A_{ij}|,k)$,
which contains up to $k$ criteria with the largest credit magnitudes.
These criteria capture the response's most pronounced strengths and weaknesses relative to the rollout group.
We compute relevance only for the selected criteria and set $\boldsymbol\omega_{ij}=\mathbf 1_{T_i}$ for $j\notin\mathcal S_i$.
This requires at most $k$ additional criterion-conditioned forward passes per response.

\paragraph{Aggregation and policy optimization.}
Finally, we form the criterion--token credits and aggregate them across criteria to obtain token-level advantages:
\begin{equation}
    [\widehat{\mathbf C}_i]_{j,t}
    =
    A_{ij}[\boldsymbol\omega_{ij}]_t,
    \qquad
    \widehat A_{i,t}
    =
    [\boldsymbol\alpha_i(\widehat{\mathbf C}_i)]_t
    =
    A_i
    +
    \textstyle\sum_{j\in\mathcal S_i}
    A_{ij}
    \left(
        [\boldsymbol\omega_{ij}]_t-1
    \right).
    \label{eq:CoCA-credit}
\end{equation}
When opposing criterion budgets are allocated differently across tokens, the resulting advantages can remain nonzero even if $A_i=0$.
We then optimize the policy using $\widehat A_{i,t}$ in place of $A_i$ in Eq.~(\ref{eq:grpo}), without introducing an auxiliary teacher-matching loss.

\subsection{Grounding Preservation and Resolution Improvement}
\label{sec:CoCA-properties}

We examine when criterion-induced relevance reduces credit allocated to irrelevant tokens compared with uniform allocation.
For each selected criterion $c_j$, let
$\varnothing\neq\mathcal T_{ij}\subsetneq\{1,\ldots,T_i\}$
denote its relevant token positions, specified without using the relevance scores.
Let $\Delta_{ij}$ denote the mean score on $\mathcal T_{ij}$ minus that on its complement, and let
$b_{ij}=\max_t d_{ij,t}-\min_t d_{ij,t}$ denote the score range.

\begin{proposition}[Grounding-preserving resolution improvement]
\label{prop:CoCA-refinement}
For any $\beta\geq0$, CoCA preserves each criterion's budget and the sign of its token-level contributions, with
$\widehat{\mathbf C}_i(\beta)\in\mathcal C_i(\mathbf a_i)$.
For each selected criterion $j\in\mathcal S_i$ with $\Delta_{ij}>0$, the following holds whenever
$0<\beta<8\Delta_{ij}/b_{ij}^2$:
\begin{equation}
    \textstyle
    \frac{1}{T_i}
    \sum_{t\notin\mathcal T_{ij}}
    \left|[\widehat{\mathbf C}_i(\beta)]_{j,t}\right|
    <
    |A_{ij}|
    \left(1-\frac{|\mathcal T_{ij}|}{T_i}\right).
    \label{eq:CoCA-refinement}
\end{equation}
The right-hand side gives the same quantity under uniform allocation.
\end{proposition}

Proposition~\ref{prop:CoCA-refinement} shows that perfect token rankings are unnecessary: a positive mean-score gap suffices to reduce off-target credit when allocation concentration is appropriately controlled.
Specifically, the bound permits stronger concentration for larger mean-score gaps but becomes more restrictive as the score range widens.
Thus, criterion-induced relevance can improve resolution while preserving the signed budgets grounded in observed outcomes.

\section{Experiments}

\subsection{Experimental Setup}
\paragraph{Datasets and Benchmarks.}
Following previous work~\citep{rar}, we evaluate CoCA on medical and scientific QA tasks, with separate policies trained on RaR-Medicine and RaR-Science, respectively.
We use the corresponding test splits for in-domain evaluation. 
To assess generalization, we further evaluate on HealthBench~\citep{healthbench}, LLMEval-Med~\citep{llmevalmed}, and MedQA~\citep{medqa} for medicine, and ResearchQA~\citep{researchqa} and GPQA-Diamond (GPQA-D)~\citep{gpqa} for science.
See Appendix~\ref{app:eval} for more dataset and evaluation details.

\paragraph{Models and Configurations.}
We use Qwen3-8B~\citep{qwen3} as the default base model and gpt-oss-120b as the default rubric judge.
Additional results across model sizes, architectures, and judges are provided in Appendix~\ref{app:model}.
We build our algorithm on the \texttt{verl} framework~\citep{verl} and train with a learning rate of $10^{-6}$, a batch size of $64$ prompts, and $G=8$ rollouts per prompt.
All experiments are conducted on $8\times$A100 GPUs, with all methods trained for $400$ steps to allow sufficient optimization.
For credit assignment, we set the criterion budget to $k=2$ for each rollout and set the allocation concentration parameter $\beta$ to $5$ for medicine and $3$ for science.
Detailed experimental protocols and training hyperparameter settings are provided in Appendix~\ref{app:train}.

\paragraph{Baselines.}
We compare CoCA with seven representative baselines:
(1) \textbf{RaR}~\citep{rar} uses GRPO with response-level rubric rewards;
(2) \textbf{OPSD}~\citep{opsd} performs on-policy self-distillation with rubrics as privileged teacher context;
(3) \textbf{RaR+OPSD} combines rubric-based GRPO with an auxiliary OPSD objective;
(4) \textbf{RuscaRL}~\citep{ruscarl} uses rubric scaffolds to guide rollout generation during training;
(5) \textbf{HeRL}~\citep{herl} guides exploration using failed rollouts and their unmet criteria as hindsight feedback;
(6) \textbf{POW3R}~\citep{pow3r} adaptively reweights rubric criteria based on within-group variation in criterion outcomes;
and (7) \textbf{CoRT}~\citep{cort} uses token-level likelihood contrasts from counterfactual rubric replay to redistribute response-level advantages.
More baseline details can be found in Appendix~\ref{app:baseline}.

\begin{table*}[!t]
\caption{\small Experimental results (\%) on medical and scientific benchmarks, using Qwen3-8B as the base model and gpt-oss-120b as the judge. Avg.\ denotes the mean score across all seven benchmarks, and the best results are highlighted in \textbf{bold}. Arrows indicate \textcolor{OliveGreen}{improvements ($\uparrow$)} or \textcolor{Maroon}{declines ($\downarrow$)} relative to the initial model.}
\centering
\renewcommand{\arraystretch}{1.3}
\setlength{\tabcolsep}{4pt}
\resizebox{\textwidth}{!}{
\begin{tabular}{l|c@{\hspace{5pt}}c@{\hspace{5pt}}c@{\hspace{4pt}}>{\columncolor{white}[4pt][8pt]}c@{\hspace{12pt}}c@{\hspace{6pt}}>{\columncolor{white}[4pt][6pt]}c@{\hspace{10pt}}>{\columncolor{white}[4pt][7pt]}c@{\hspace{7pt}}|@{\hspace{7pt}}>{\columncolor{white}[7pt][4pt]}c}
\toprule
\multirow{2}{*}{\textbf{Method}}
& \multicolumn{4}{c@{\hspace{12pt}}}{\textbf{Medicine}}
& \multicolumn{3}{@{}c@{\hspace{7pt}}|@{\hspace{7pt}}}{\textbf{Science}}
& \multirow{2}{*}{\textbf{Avg.}} \\
\cmidrule(l{4pt}r{12pt}){2-5}\cmidrule(l{0pt}r{14.4pt}){6-8}
& \textbf{RaR-Medicine}
& \textbf{HealthBench}
& \textbf{LLMEval-Med}
& \hspace*{-1pt}\textbf{MedQA}\hspace*{1pt}
& \textbf{RaR-Science}
& \textbf{ResearchQA}
& \hspace*{-2pt}\textbf{GPQA-D}\hspace*{2pt}
& \\
\midrule
Initial
& 54.2 \basex{0.0}
& 41.9 \basex{0.0}
& 51.1 \basex{0.0}
& 72.0 \basex{0.0}
& 67.6 \basex{0.0}
& 50.7 \basex{0.0}
& 49.7 \basex{0.0}
& 55.3 \basex{0.0} \\
RaR
& 65.0 \uprise{10.8}
& 47.5 \uprise{5.6}
& 55.4 \uprise{4.3}
& 73.8 \uprise{1.8}
& 72.4 \uprise{4.8}
& 56.6 \uprise{5.9}
& 51.6 \uprise{1.9}
& 60.3 \uprise{5.0} \\
OPSD
& 44.7 \down{9.5}
& 30.6 \down{11.3}
& 48.8 \down{2.3}
& 69.0 \down{3.0}
& 65.5 \down{2.1}
& 47.5 \down{3.2}
& 46.6 \down{3.1}
& 50.4 \down{4.9} \\
RaR+OPSD
& 61.8 \uprise{7.6}
& 42.1 \uprise{0.2}
& 53.4 \uprise{2.3}
& 72.1 \uprise{0.1}
& 67.9 \uprise{0.3}
& 50.6 \down{0.1}
& 48.8 \down{0.9}
& 56.7 \uprise{1.4} \\
RuscaRL
& 58.4 \uprise{4.2}
& 45.3 \uprise{3.4}
& 55.6 \uprise{4.5}
& 73.8 \uprise{1.8}
& 67.9 \uprise{0.3}
& 53.1 \uprise{2.4}
& 44.4 \down{5.3}
& 56.9 \uprise{1.6} \\
HeRL
& 67.5 \uprise{13.3}
& 49.8 \uprise{7.9}
& 56.7 \uprise{5.6}
& 73.3 \uprise{1.3}
& 72.7 \uprise{5.1}
& 58.4 \uprise{7.7}
& 51.6 \uprise{1.9}
& 61.4 \uprise{6.1} \\
POW3R
& 66.7 \uprise{12.5}
& 49.3 \uprise{7.4}
& 55.3 \uprise{4.2}
& 72.4 \uprise{0.4}
& 72.9 \uprise{5.3}
& 58.8 \uprise{8.1}
& 52.0 \uprise{2.3}
& 61.1 \uprise{5.8} \\
CoRT
& 67.1 \uprise{12.9}
& 48.0 \uprise{6.1}
& 55.7 \uprise{4.6}
& 75.9 \uprise{3.9}
& 72.6 \uprise{5.0}
& 58.1 \uprise{7.4}
& 47.7 \down{2.0}
& 60.7 \uprise{5.4} \\
\rowcolor{gray!15}
\textbf{CoCA (ours)}
& \textbf{69.8} \uprise{15.6}
& \textbf{54.3} \uprise{12.4}
& \textbf{61.8} \uprise{10.7}
& \textbf{77.6} \uprise{5.6}
& \textbf{74.2} \uprise{6.6}
& \textbf{59.6} \uprise{8.9}
& \textbf{54.4} \uprise{4.7}
& \textbf{64.5} \uprise{9.2} \\
\bottomrule
\end{tabular}
}
\label{tb:main}
\end{table*}

\subsection{Main Results}
\paragraph{CoCA achieves SOTA performance.}
As shown in Table~\ref{tb:main}, CoCA achieves the best performance on all seven benchmarks, reaching an average score of $64.5\%$ and improving over RaR, the standard rubric-based RL baseline, by $4.2\%$.
These gains highlight the importance of translating criterion-level feedback into targeted token-level updates for effective rubric-based learning.
Notably, RaR + OPSD even underperforms RaR on every benchmark and trails CoCA by $7.8\%$ on average, showing that directly optimizing an auxiliary rubric-conditioned distillation objective can undermine reward-driven learning.
Moreover, CoCA surpasses CoRT, which also targets credit assignment in rubric-based RL, by $3.8\%$ on average.
This gap suggests that resolving credit for individual criteria is more effective than reweighting response-level advantages across tokens, as done in CoRT.
Collectively, CoCA enables more effective rubric-based RL through criterion-aware credit assignment.

\paragraph{CoCA enables more targeted rubric-based learning.}
To distinguish criterion-level gains and losses obscured by aggregate scores, we measure correction and regression rates relative to the initial model on the same test questions.
As shown in Figure~\ref{fig:criterion}, CoCA achieves the highest correction and lowest regression rates in both domains.
Compared with RaR, correction rates increase from $36.7\%$ to $46.3\%$ in medicine and from $26.5\%$ to $32.0\%$ in science.
Meanwhile, regression rates decrease from $10.8\%$ to $8.6\%$ and from $9.5\%$ to $7.6\%$, respectively.
By comparison, RaR+OPSD exhibits the opposite pattern, combining fewer corrections with more regressions than RaR in both domains.
This contrast highlights CoCA's more targeted rubric-based learning: it corrects unmet criteria more effectively while better preserving those already satisfied.

\begin{table*}[!t]
\caption{\small Ablation results (\%) for credit allocation mechanisms. Arrows indicate changes relative to CoCA.}
\centering
\renewcommand{\arraystretch}{1.3}
\setlength{\tabcolsep}{4pt}
\resizebox{\textwidth}{!}{
\begin{tabular}{l@{\hspace{7pt}}|c@{\hspace{5pt}}c@{\hspace{5pt}}c@{\hspace{4pt}}>{\columncolor{white}[4pt][8pt]}c@{\hspace{12pt}}c@{\hspace{6pt}}>{\columncolor{white}[4pt][6pt]}c@{\hspace{10pt}}>{\columncolor{white}[4pt][7pt]}c@{\hspace{7pt}}|@{\hspace{7pt}}>{\columncolor{white}[7pt][4pt]}c}
\toprule
\multirow{2}{*}{\textbf{Method}}
& \multicolumn{4}{c@{\hspace{12pt}}}{\textbf{Medicine}}
& \multicolumn{3}{@{}c@{\hspace{7pt}}|@{\hspace{7pt}}}{\textbf{Science}}
& \multirow{2}{*}{\textbf{Avg.}} \\
\cmidrule(l{4pt}r{12pt}){2-5}
\cmidrule(l{0pt}r{14.4pt}){6-8}
& \textbf{RaR-Medicine}
& \textbf{HealthBench}
& \textbf{LLMEval-Med}
& \hspace*{-1pt}\textbf{MedQA}\hspace*{1pt}
& \textbf{RaR-Science}
& \textbf{ResearchQA}
& \hspace*{-2pt}\textbf{GPQA-D}\hspace*{2pt}
& \\
\midrule
\textbf{CoCA}
& \textbf{69.8} \basex{0.0}
& \textbf{54.3} \basex{0.0}
& \textbf{61.8} \basex{0.0}
& \textbf{77.6} \basex{0.0}
& \textbf{74.2} \basex{0.0}
& \textbf{59.6} \basex{0.0}
& \textbf{54.4} \basex{0.0}
& \textbf{64.5} \basex{0.0}\\
\midrule
Uniform
& 65.0 \down{4.8}
& 47.5 \down{6.8}
& 55.4 \down{6.4}
& 73.8 \down{3.8}
& 72.4 \down{1.8}
& 56.6 \down{3.0}
& 51.6 \down{2.8}
& 60.3 \down{4.2} \\
Shuffled
& 50.7 \down{19.1}
& 39.4 \down{14.9}
& 53.1 \down{8.7}
& 69.2 \down{8.4}
& 65.2 \down{9.0}
& 50.5 \down{9.1}
& 43.8 \down{10.6}
& 53.1 \down{11.4} \\
Shared
& 67.7 \down{2.1}
& 50.1 \down{4.2}
& 57.6 \down{4.2}
& 73.9 \down{3.7}
& 72.7 \down{1.5}
& 57.7 \down{1.9}
& 52.0 \down{2.4}
& 61.7 \down{2.8} \\
\bottomrule
\end{tabular}
}
\label{tb:ablation_1}
\vspace{-2pt}
\end{table*}

\begin{figure*}[!t]
    \centering
    \includegraphics[width=\textwidth]{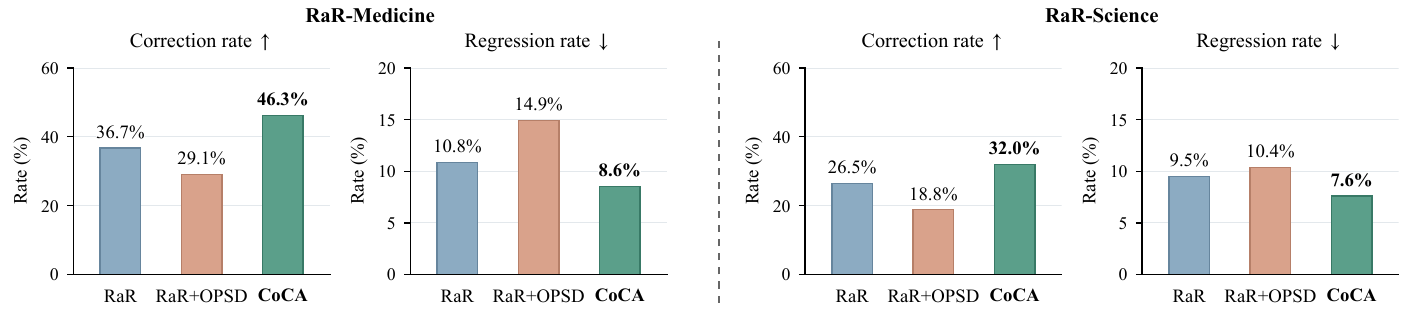}
    \caption{\small Criterion-level correction ($\uparrow$) and regression ($\downarrow$) rates on RaR-Medicine and RaR-Science datasets. Correction measures the fraction of initially unmet criteria that become satisfied after training, while regression measures the fraction of initially satisfied criteria that become unmet.}
    \label{fig:criterion}
    \vspace{-6pt}
\end{figure*}

\subsection{Further Analysis}
\paragraph{Analysis of paradigm limitations.}
To investigate optimization challenges in rubric-based RL and OPSD, we quantify two types of signal conflict during training: 
(i) \textit{criterion credit conflict}, where a rollout's response-level advantage and at least one criterion credit have opposite signs; 
(ii) \textit{distillation misalignment}, where OPSD and outcome-grounded learning yield negatively aligned logit gradients for the same criterion--token pair.
As shown in Figure~\ref{fig:conflicts}, criterion credit conflict reaches $54.6\%$ in medicine and $50.4\%$ in science, while distillation misalignment reaches $49.0\%$ and $48.8\%$, respectively.
These frequent conflicts indicate that rubric-based RL can reinforce criterion-violating content or suppress criterion-satisfying content, while OPSD often conflicts with outcome-grounded learning at the logit level.
This sheds light on the performance limitations of both paradigms.

\paragraph{Ablation of credit allocation mechanisms.}
Table~\ref{tb:ablation_1} compares CoCA with three allocation variants.
\textit{Uniform} assigns equal weights to all response tokens, reducing the average score by $4.2\%$ relative to CoCA.
This decrease supports the benefit of resolving criterion-level credit through nonuniform token allocation.
\textit{Shuffled} applies the same random permutation of token positions to the allocation weights of all criteria within each response.
This preserves the full distribution of aggregated token advantages while disrupting their alignment with the response tokens.
Its average score drops by $11.4\%$, highlighting the importance of assigning credit to appropriate token positions.
Furthermore, to examine the role of criterion-specific allocation, we introduce \textit{Shared}, which redistributes the response-level advantage using token weights averaged across the selected criteria.
The resulting $2.8\%$ decrease in average score supports using distinct allocation profiles for different criteria.
These results demonstrate that CoCA achieves more effective credit resolution by allocating each criterion's credit according to its token relevance.

\paragraph{Ablation of relevance measures.}
To examine how relevance estimation affects credit allocation, we compare JSD with alternative measures in Table~\ref{tb:ablation_2}.
Replacing JSD with forward KL and reverse KL reduces the average score by $5.3\%$ and $2.8\%$, respectively.
Compared with these asymmetric, unbounded divergences, JSD treats both distributions symmetrically and limits extreme relevance scores, which offers a plausible explanation for its advantage.
We further consider the signed log-probability ratio, which favors generated tokens whose probabilities increase under criterion conditioning.
This substitution reduces the average score by $10.3\%$, showing that full-distribution changes provide more informative allocation signals than the directional probability shift of the generated token alone.
These comparisons support JSD as an effective relevance measure for credit allocation.

\begin{table*}[!t]
\caption{\small Ablation results (\%) for relevance measures. Arrows indicate changes relative to CoCA.}
\centering
\renewcommand{\arraystretch}{1.3}
\setlength{\tabcolsep}{4pt}
\resizebox{\textwidth}{!}{
\begin{tabular}{l@{\hspace{7pt}}|c@{\hspace{5pt}}c@{\hspace{5pt}}c@{\hspace{4pt}}>{\columncolor{white}[4pt][8pt]}c@{\hspace{12pt}}c@{\hspace{6pt}}>{\columncolor{white}[4pt][6pt]}c@{\hspace{10pt}}>{\columncolor{white}[4pt][7pt]}c@{\hspace{7pt}}|@{\hspace{7pt}}>{\columncolor{white}[7pt][4pt]}c}
\toprule
\multirow{2}{*}{\textbf{Measure}}
& \multicolumn{4}{c@{\hspace{12pt}}}{\textbf{Medicine}}
& \multicolumn{3}{@{}c@{\hspace{7pt}}|@{\hspace{7pt}}}{\textbf{Science}}
& \multirow{2}{*}{\textbf{Avg.}} \\
\cmidrule(l{4pt}r{12pt}){2-5}
\cmidrule(l{0pt}r{14.4pt}){6-8}
& \textbf{RaR-Medicine}
& \textbf{HealthBench}
& \textbf{LLMEval-Med}
& \hspace*{-1pt}\textbf{MedQA}\hspace*{1pt}
& \textbf{RaR-Science}
& \textbf{ResearchQA}
& \hspace*{-2pt}\textbf{GPQA-D}\hspace*{2pt}
& \\
\midrule
\textbf{JSD (CoCA)}
& \textbf{69.8} \basex{0.0}
& \textbf{54.3} \basex{0.0}
& \textbf{61.8} \basex{0.0}
& \textbf{77.6} \basex{0.0}
& \textbf{74.2} \basex{0.0}
& \textbf{59.6} \basex{0.0}
& \textbf{54.4} \basex{0.0}
& \textbf{64.5} \basex{0.0}\\
\midrule
Forward KL
& 60.8 \down{9.0} & 48.8 \down{5.5} & 57.3 \down{4.5} & 73.5 \down{4.1} & 70.3 \down{3.9} & 54.5 \down{5.1} & 49.1 \down{5.3} & 59.2 \down{5.3} \\
Reverse KL
& 66.7 \down{3.1} & 49.7 \down{4.6} & 57.1 \down{4.7} & 75.9 \down{1.7} & 72.9 \down{1.3} & 57.2 \down{2.4} & 52.4 \down{2.0} & 61.7 \down{2.8} \\
Log-ratio
& 53.4 \down{16.4} & 41.4 \down{12.9} & 50.8 \down{11.0} & 70.0 \down{7.6} & 66.9 \down{7.3} & 50.7 \down{8.9} & 46.3 \down{8.1} & 54.2 \down{10.3} \\
\bottomrule
\end{tabular}
}
\label{tb:ablation_2}
\end{table*}

\begin{figure*}[t]
    \centering
    \begin{subfigure}[t]{0.32\textwidth}
        \centering
        \vspace{0pt}
        \includegraphics[width=\linewidth]{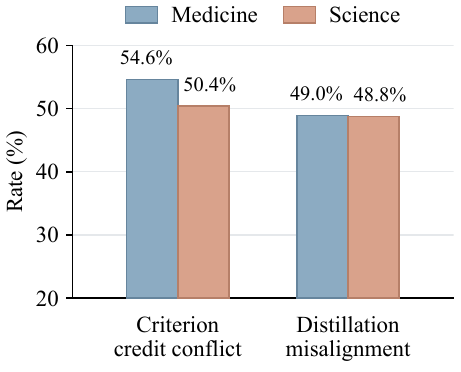}
        \caption{Signal conflict rates}
        \label{fig:conflicts}
    \end{subfigure}
    \hfill
    \begin{subfigure}[t]{0.65\textwidth}
        \centering
        \vspace{0pt}
        \includegraphics[width=\linewidth]{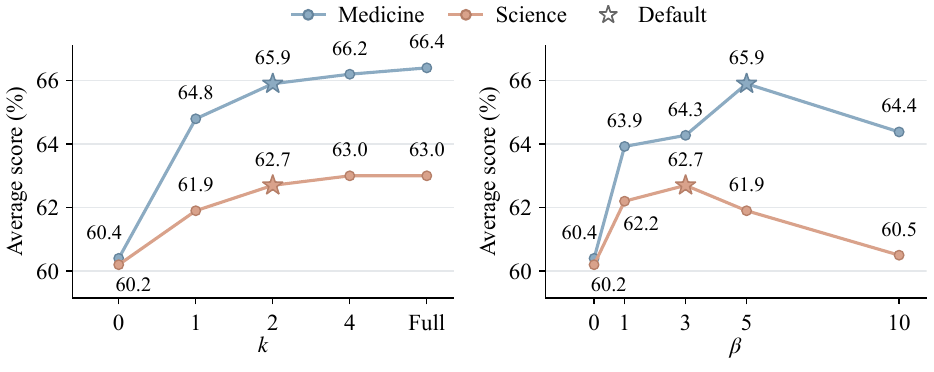}
        \caption{Sensitivity to budget $k$ and allocation concentration $\beta$}
        \label{fig:sensitivity}
    \end{subfigure}
    \caption{\small
    (a) Rates of two signal conflicts: response-level advantages versus criterion-level credits, and distillation gradients versus outcome-grounded gradients.
    (b) Effects of $k$ and $\beta$ on average scores in each domain.
    }
    \vspace{-6pt}
    \label{fig:sensitivity_conflicts}
\end{figure*}

\paragraph{Hyperparameter sensitivity.}
We study CoCA's sensitivity to two key hyperparameters: the criterion budget $k$ and allocation concentration $\beta$.
As shown in Figure~\ref{fig:sensitivity}, even $k=1$ improves average scores over uniform allocation ($k=0$) by $4.4\%$ in medicine and $1.7\%$ in science.
Beyond $k=2$, gains are marginal: selecting all criteria adds only $0.5\%$ and $0.3\%$, respectively.
These results indicate that prioritizing criteria with the largest absolute credits captures most of the benefit of full allocation.
Thus, $k=2$ offers a practical balance between performance and computation, requiring at most two additional forward passes per rollout.
For $\beta$, performance peaks at $5$ in medicine and $3$ in science.
This pattern favors moderate concentration, suggesting that small values dilute updates to relevant tokens, while large values can neglect other tokens needed for criterion satisfaction.

\paragraph{Case study of criterion--token relevance.}
To qualitatively assess whether criterion-induced distribution shifts capture token-level relevance, Figure~\ref{fig:illustration} presents a medical response under a criterion concerning when to seek medical attention.
Higher relevance scores concentrate on the parts of the response that provide the corresponding recommendation and describe the conditions under which it applies, while the general care advice receives much lower scores.
Notably, this contrast indicates that the relevance signal captures criterion-specific semantics rather than general topical relevance.
Such precise localization enables CoCA to perform targeted credit assignment at the token level.


\begin{figure*}[!t]
    \centering
    \includegraphics[width=\textwidth]{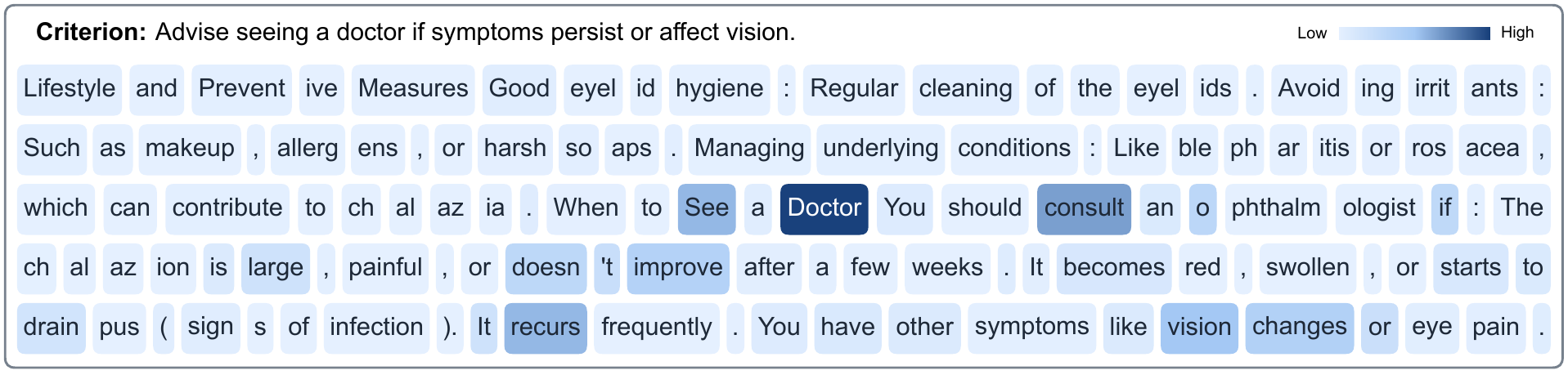}
    \caption{\small Visualization of criterion--token relevance in a medical response. Darker colors indicate higher JSD between next-token distributions with and without criterion conditioning.}
    \label{fig:illustration}
    \vspace{-6pt}
\end{figure*}

\section{Related Work}
\label{sec:related-work}

\paragraph{Rubrics as rewards.}
Rubric-based RL evaluates open-ended responses against explicit criteria, typically aggregating the resulting judgments into a scalar reward~\citep{rar,checklists,rubricanchors}.
To strengthen the learning signal, existing methods refine criteria as the policy evolves~\citep{drtulu,care}, adjust criterion weights~\citep{pow3r}, or alternate optimization across criterion categories~\citep{arlrr}.
However, these methods retain response-level rewards.
To provide finer-grained supervision, another line of work localizes criterion feedback to reasoning steps or token spans~\citep{rtt,srar,draco,rcca}.
Their reliance on explicit attribution signals or separately trained relevance models adds supervision and computational overhead that complicates scaling.
More fundamentally, these advances do not by themselves ensure that each criterion's signed contribution to the response-level advantage is preserved, motivating our focus on criterion-aware credit assignment.

\paragraph{Rubrics as guidance.}
Beyond defining rewards, rubrics also provide training-time guidance for policy learning.
One line of work guides rollout generation with rubric instructions or feedback to promote exploration~\citep{ruscarl,rgr,herl,riserl}.
However, training on rubric-guided prefixes creates a mismatch with unguided inference, introducing exposure bias and a risk of compounding errors.
To avoid this mismatch, another line adopts OPSD, which uses a rubric-informed teacher to provide token-level supervision on rollouts sampled from the unguided student~\citep{rcsd,rgsd,rupi,llmcoach,papergym}.
Such distillation can serve as a standalone optimization objective or be combined with GRPO as an auxiliary objective.
Although these methods provide dense supervision, they introduce a separate imitation objective that may conflict with reward-driven optimization.
We therefore use privileged rubric information as a relevance cue for outcome-grounded credit allocation.

\paragraph{Credit assignment.}
In classical RL, temporal credit is assigned by estimating advantages from value functions~\citep{gae}, redistributing delayed rewards~\citep{rudder}, or inferring action contributions from subsequent outcomes~\citep{hca,cocoa}.
For LLMs trained with RL on verifiable tasks, methods infer local rewards or advantages using implicit reward models~\citep{prime} or shared rollout prefixes~\citep{tempo}, while others concentrate policy updates on high-entropy tokens~\citep{entropytokens}.
Rubric-based RL poses an additional challenge: a response may satisfy some criteria while violating others, so redistributing an aggregate advantage alone cannot distinguish their opposing contributions.
CoRT derives positive token weights through counterfactual rubric replay~\citep{cort}, but applies them to a single response-level advantage, tying all token advantages to its sign.
Hence, the key is to preserve each criterion's signed contribution during token-level credit assignment.

\section{Conclusion}
In this work, we propose CoCA, a criterion-aware credit assignment framework for rubric-based RL.
Our core insight is that effective learning from rubrics requires preserving the distinct contributions of individual criteria and assigning their credit to relevant parts of a response.
CoCA realizes this principle by deriving signed criterion credits from observed outcomes and using privileged rubric information to guide their allocation across tokens.
This translates criterion-level feedback into fine-grained, outcome-grounded policy updates.
Experiments across medicine and science show CoCA outperforms standard rubric-based RL by $4.2\%$, with higher correction and lower regression rates.
We hope our work encourages greater attention to credit assignment as a bridge between structured evaluation and policy optimization, paving the way toward more effective RL for open-ended tasks.

\subsection*{AI use statement}
We used generative AI for literature retrieval and discovery, coding assistance, and improving the grammar, clarity, and readability of author-written text.
In our experimental pipeline, gpt-oss-120b and DeepSeek-V4.1-Flash served as rubric-based judges, while GPT-6 Astra generated word-level annotations for the criterion--token relevance analysis. These experimental uses and their validation procedures are described in the appendix.
The authors reviewed all AI-assisted work and take full responsibility for the content of the paper.

\subsection*{Ethics statement}
This work develops and evaluates a criterion-aware reinforcement learning method using only publicly available datasets and benchmarks in medical and scientific QA.
The medical benchmarks serve as research testbeds for assessing open-ended response quality and generalization. 
Accordingly, our results should not be interpreted as evidence of clinical safety or suitability for clinical use.

\subsection*{Reproducibility statement}
Complete proofs of the theoretical results are provided in Appendix~\ref{app:proofs}.
We document the datasets and data splits, evaluation protocols, prompts, training and implementation settings, and baseline reproduction details in Appendix~\ref{app:exp_detail}.
The source code, configuration files, and evaluation scripts are included in the anonymous supplementary material.

\bibliography{iclr2027_conference}
\bibliographystyle{iclr2027_conference}

\appendix
\raggedbottom
\section{Theoretical Proofs}
\label{app:proofs}

\subsection{Proof of Proposition~\ref{prop:rubric-rl-collapse}}
\label{app:proof-1}

\begin{cocarestated}{prop:rubric-rl-collapse}{Credit resolution collapse in rubric-based RL}
For any oracle resolution $\mathbf C_i^*\in\mathcal C_i(\mathbf a_i)$, the token credit induced by rubric-based RL and the corresponding projection error satisfy
\begin{equation*}
    \left(\mathbf 1_M^\top\mathbf a_i\right)\mathbf 1_{T_i}
    =
    A_i\mathbf 1_{T_i}
    =
    \boldsymbol\alpha_i
    \!\left(\Pi_{\mathcal U_i}(\mathbf C_i^*)\right),
    \quad
    \tfrac{1}{T_i}
    \left\|
        \mathbf C_i^*-\Pi_{\mathcal U_i}(\mathbf C_i^*)
    \right\|_F^2
    =
    {\textstyle\sum_{j=1}^{M}}
    \operatorname{Var}_{t}
    \left([\mathbf C_i^*]_{j,t}\right).
    \tag{\ref{eq:rubric-rl-collapse}}
\end{equation*}
Moreover, the linear map $\mathbf a\mapsto(\mathbf 1_M^\top\mathbf a)\mathbf 1_{T_i}$ has rank one and kernel $\mathbf 1_M^\perp$.
Consequently, any nonzero criterion budget vector $\mathbf a_i\in\mathbf 1_M^\perp$ yields a zero token-credit vector, while any position $t$ satisfying $A_i[\boldsymbol\alpha_i(\mathbf C_i^*)]_t<0$ is assigned credit whose sign is opposite to its oracle token credit.
\end{cocarestated}

\begin{proof}
Fix a rollout $i$ and an oracle resolution $\mathbf C_i^*\in\mathcal C_i(\mathbf a_i)$.
Recall that $\mathcal U_i=\{\mathbf u\mathbf 1_{T_i}^{\top}:\mathbf u\in\mathbb R^M\}$ consists of matrices whose rows are constant across tokens.
For any $\mathbf C\in\mathbb R^{M\times T_i}$, write $\mu_j=T_i^{-1}\sum_t C_{j,t}$ for its $j$-th row mean.
Centering each row gives
\begin{equation*}
\begin{aligned}
    \|\mathbf C-\mathbf u\mathbf 1_{T_i}^{\top}\|_F^2
    &=\sum_{j=1}^{M}\sum_{t=1}^{T_i}
      \bigl((C_{j,t}-\mu_j)+(\mu_j-u_j)\bigr)^2\\
    &=\sum_{j=1}^{M}\sum_{t=1}^{T_i}(C_{j,t}-\mu_j)^2
      +T_i\sum_{j=1}^{M}(u_j-\mu_j)^2.
\end{aligned}
\end{equation*}
The cross terms vanish because $\sum_t(C_{j,t}-\mu_j)=0$ for every row.
The first term is independent of $\mathbf u$, while the second is uniquely minimized at $u_j=\mu_j$.
Thus, projection replaces each row by its mean.
Since $T_i^{-1}\mathbf C_i^*\mathbf 1_{T_i}=\mathbf a_i$ by the budget constraint,
\begin{equation*}
    \Pi_{\mathcal U_i}(\mathbf C)
    =\left(\frac{1}{T_i}\mathbf C\mathbf 1_{T_i}\right)\mathbf 1_{T_i}^{\top},
    \qquad
    \Pi_{\mathcal U_i}(\mathbf C_i^*)=\mathbf a_i\mathbf 1_{T_i}^{\top}.
\end{equation*}
Using $\boldsymbol\alpha_i(\mathbf C)=\mathbf C^\top\mathbf 1_M$ and
$\mathbf 1_M^\top\mathbf a_i=A_i$, we obtain
\begin{equation*}
    \boldsymbol\alpha_i\!\left(\Pi_{\mathcal U_i}(\mathbf C_i^*)\right)
    =(\mathbf a_i\mathbf 1_{T_i}^{\top})^\top\mathbf 1_M
    =(\mathbf 1_M^\top\mathbf a_i)\mathbf 1_{T_i}
    =A_i\mathbf 1_{T_i}.
\end{equation*}
This is the token-credit vector used by rubric-based RL, regardless of the oracle's allocation within each row.
Moreover, because the $j$-th row of $\mathbf C_i^*$ has mean $A_{ij}$,
\begin{equation*}
    \frac{1}{T_i}\left\|\mathbf C_i^*-\Pi_{\mathcal U_i}(\mathbf C_i^*)\right\|_F^2
    =\sum_{j=1}^{M}\frac{1}{T_i}\sum_{t=1}^{T_i}
      \bigl([\mathbf C_i^*]_{j,t}-A_{ij}\bigr)^2
    =\sum_{j=1}^{M}\operatorname{Var}_{t}\!\left([\mathbf C_i^*]_{j,t}\right),
\end{equation*}
where $\operatorname{Var}_t$ uses the uniform distribution over the $T_i$ token positions.
Thus, the projection discards exactly the within-criterion variation measured by this identity.

Finally, the budget-to-token map is represented by
$\mathbf L_i=\mathbf 1_{T_i}\mathbf 1_M^\top$.
Both factors in this outer product are nonzero, so $\operatorname{rank}(\mathbf L_i)=1$.
Furthermore,
\begin{equation*}
    \mathbf L_i\mathbf a=0
    \quad\Longleftrightarrow\quad
    \mathbf 1_M^\top\mathbf a=0,
    \qquad
    \ker\mathbf L_i=\mathbf 1_M^\perp.
\end{equation*}
Every nonzero vector in this kernel gives zero token credit despite its nonzero criterion contributions.
Since rubric-based RL assigns $A_i$ at every position, the condition
$A_i[\boldsymbol\alpha_i(\mathbf C_i^*)]_t<0$ means that the assigned and oracle token credits are nonzero and have opposite signs.
\end{proof}

\subsection{Proof of Proposition~\ref{prop:opsd-grounding}}
\label{app:proof-2}

\begin{cocarestated}{prop:opsd-grounding}{OPSD matching can reduce reward}
Consider the next-token decision at $h_{i,t}$ as a one-step contextual bandit with finite action space $\mathcal V$, a reward function $R:\mathcal V\to\mathbb R$, and a tabular softmax policy.
Let $\theta'$ be obtained by a natural-gradient descent step of size $\eta>0$ on the KL term of Eq.~(\ref{eq:opsd}), using the local Fisher information and holding the self-teacher fixed.
Then, as $\eta\to0$,
\begin{equation*}
    J(\theta';h_{i,t})-J(\theta;h_{i,t})
    =
    \eta\,
    \operatorname{Cov}_{a\sim\pi_\theta(\cdot\mid h_{i,t})}
    \left(
        R(a),
        r_{i,t}(a;\xi)
    \right)
    +
    \mathcal O(\eta^2).
    \tag{\ref{eq:opsd-grounding}}
\end{equation*}
When $|\mathcal V|\geq3$, there exist instances for which $J(\theta;h_{i,t}(\xi))\!>\!J(\theta;h_{i,t})$ while the covariance in Eq.~(\ref{eq:opsd-grounding}) is strictly negative, implying $J(\theta';h_{i,t})\!<\!J(\theta;h_{i,t})$ for sufficiently small $\eta>0$.
\end{cocarestated}

\begin{proof}
We first derive the local natural-gradient update and its first-order effect on reward.
We then construct a higher-reward teacher for which this effect is negative.
Throughout the proof, the action space is finite and all logarithms are natural.

\emph{Local matching update.}
At the fixed context $h=h_{i,t}$, write $p$ for the student distribution, $q$ for the detached teacher distribution, and $\mathbf r$ for their matching signal:
\begin{equation*}
    p_a=\pi_\theta(a\mid h),\qquad
    q_a=\pi_\theta(a\mid h_{i,t}(\xi)),\qquad
    r_a=\log\frac{q_a}{p_a}.
\end{equation*}
These quantities are evaluated before the update, and $q$ is held fixed when differentiating the matching objective.
Finite softmax logits ensure $p_a,q_a>0$, so every $r_a$ is finite.
Let $n=|\mathcal V|$ and let $\boldsymbol\vartheta\in\mathbb R^n$ denote the student logits at this fixed context.
The local reverse-KL loss is
\begin{equation*}
    \ell(\boldsymbol\vartheta)
    =D_{\mathrm{KL}}(p\|q)
    =\sum_{a\in\mathcal V}p_a\log\frac{p_a}{q_a},
    \qquad
    p_a=\frac{e^{\vartheta_a}}{\sum_{b\in\mathcal V}e^{\vartheta_b}}.
\end{equation*}
Writing $\mathbf e_a$ for the coordinate vector of action $a$, softmax differentiation gives
\begin{equation*}
    \nabla_{\boldsymbol\vartheta}\log p_a=\mathbf e_a-p,
    \qquad
    \frac{\partial p_a}{\partial\vartheta_b}
    =p_a\bigl(\mathbf 1\{a=b\}-p_b\bigr).
\end{equation*}
Consequently, the local Fisher matrix equals the softmax Jacobian:
\begin{equation*}
    F=\sum_a p_a(\mathbf e_a-p)(\mathbf e_a-p)^\top
      =\operatorname{diag}(p)-pp^\top.
\end{equation*}
Differentiating the loss through the student probabilities yields
\begin{equation*}
\begin{aligned}
    \frac{\partial\ell}{\partial\vartheta_b}
    &=\sum_a\frac{\partial p_a}{\partial\vartheta_b}
      \left(\log\frac{p_a}{q_a}+1\right),\\
    \nabla_{\boldsymbol\vartheta}\ell
    &=F(\log p-\log q+\mathbf 1_n)=-F\mathbf r,
\end{aligned}
\end{equation*}
Here $\mathbf r=(r_a)_{a\in\mathcal V}$, vector logarithms are componentwise, and $F\mathbf 1_n=0$ gives the last equality.

The Fisher matrix is singular because softmax is invariant to a common shift of all logits.
Its nullspace can be identified from the following quadratic form, valid for every $\mathbf v\in\mathbb R^n$:
\begin{equation*}
    \mathbf v^\top F\mathbf v
    =\sum_a p_a v_a^2-\left(\sum_a p_a v_a\right)^2
    =\operatorname{Var}_{a\sim p}(v_a).
\end{equation*}
Since $p$ has full support, this variance vanishes exactly when every $v_a$ has the same value.
Thus $\ker F=\operatorname{span}(\mathbf 1_n)$, and its Moore--Penrose pseudoinverse satisfies
\begin{equation*}
    F^\dagger F=I_n-\frac1n\mathbf 1_n\mathbf 1_n^\top.
\end{equation*}
The natural-gradient descent step, with $F$ evaluated at the original student policy, is therefore
\begin{equation*}
    \boldsymbol\vartheta'
    =\boldsymbol\vartheta-\eta F^\dagger\nabla_{\boldsymbol\vartheta}\ell
      =\boldsymbol\vartheta+\eta F^\dagger F\mathbf r
    =\boldsymbol\vartheta+\eta(\mathbf r-\kappa\mathbf 1_n),
    \qquad \kappa=\frac1n\sum_a r_a.
\end{equation*}
The common shift $-\eta\kappa$ cancels in softmax, giving the exact probability update
\begin{equation*}
    p'_a
    =\frac{e^{\vartheta_a+\eta(r_a-\kappa)}}
           {\sum_b e^{\vartheta_b+\eta(r_b-\kappa)}}
    =\frac{p_a e^{\eta r_a}}{Z(\eta)},
    \qquad
    Z(\eta)=\sum_b p_b e^{\eta r_b}.
\end{equation*}

\emph{First-order changes in reward and matching loss.}
Let $\bar r=\sum_a p_a r_a$ be the student-weighted mean of the matching signal.
This differs from the unweighted mean $\kappa$ used only to fix the logit shift above.
Expanding the normalizing factor and its reciprocal at $\eta=0$ gives
\begin{equation*}
    Z(\eta)=1+\eta\bar r+\mathcal O(\eta^2),
    \qquad
    Z(\eta)^{-1}=1-\eta\bar r+\mathcal O(\eta^2).
\end{equation*}
Combining these expansions with $e^{\eta r_a}=1+\eta r_a+\mathcal O(\eta^2)$ yields
\begin{equation*}
    p'_a
    =p_a\bigl(1+\eta r_a+\mathcal O(\eta^2)\bigr)
      \bigl(1-\eta\bar r+\mathcal O(\eta^2)\bigr)
    =p_a+\eta p_a(r_a-\bar r)+\mathcal O(\eta^2).
\end{equation*}
For fixed $p,q$, every log-ratio is finite and $Z(0)=1$, so the updated probabilities are smooth near $\eta=0$.
The finite action space allows a common bound on their second-order remainders.
Since $\mathcal V$ is finite and $R$ is real-valued,
$\max_{a\in\mathcal V}|R(a)|<\infty$.
Multiplying the probability expansion by $R(a)$ and summing
therefore gives
\begin{equation*}
\begin{aligned}
    J(\theta';h)-J(\theta;h)
    &=\sum_a R(a)(p'_a-p_a)\\
    &=\eta\left(\sum_a p_a R(a)r_a
      -\left(\sum_a p_a R(a)\right)\left(\sum_a p_a r_a\right)\right)
      +\mathcal O(\eta^2)\\
    &=\eta\operatorname{Cov}_{a\sim p}(R(a),r_a)+\mathcal O(\eta^2).
\end{aligned}
\end{equation*}
Restoring $h=h_{i,t}$ and $r_a=r_{i,t}(a;\xi)$ proves the first assertion.

The same step also decreases the local matching loss for sufficiently small positive $\eta$ whenever $\mathbf r$ is nonconstant.
Indeed, a Taylor expansion of the loss along the update direction gives
\begin{equation*}
\begin{aligned}
\ell(\boldsymbol\vartheta')-\ell(\boldsymbol\vartheta)
    &=-\eta(\nabla_{\boldsymbol\vartheta}\ell)^\top
      F^\dagger\nabla_{\boldsymbol\vartheta}\ell+\mathcal O(\eta^2)\\
&=-\eta\mathbf r^\top FF^\dagger F\mathbf r+\mathcal O(\eta^2)
=-\eta\operatorname{Var}_{a\sim p}(r_a)+\mathcal O(\eta^2),
\end{aligned}
\end{equation*}
using $FF^\dagger F=F$.
A negative reward covariance is possible only when $\mathbf r$ is nonconstant.
Thus, in that case, small matching steps decrease both the distillation loss and expected reward.

\emph{A higher-reward teacher with a harmful matching direction.}
Consider three actions with rewards $R$, student probabilities $p$, and teacher probabilities $q$, listed in the same order:
\begin{equation*}
    R=\left(1,\frac12,0\right),\qquad
    p=\left(\frac1{20},\frac9{20},\frac12\right),\qquad
    q=\left(\frac3{10},\frac1{20},\frac{13}{20}\right).
\end{equation*}
Both distributions have full support and sum to one.
Their expected rewards are
\begin{equation*}
    \mathbb E_p[R]=\frac1{20}+\frac9{20}\cdot\frac12=\frac{11}{40},
    \qquad
    \mathbb E_q[R]=\frac3{10}+\frac1{20}\cdot\frac12=\frac{13}{40}.
\end{equation*}
The teacher therefore has strictly higher expected reward.
Nevertheless, its matching signal $\mathbf r=(\log6,-\log9,\log(13/10))$ has negative covariance with reward under the student:
\begin{equation*}
\begin{aligned}
    \operatorname{Cov}_p(R,r)
    &=\sum_{a=1}^{3}p_a\left(R(a)-\frac{11}{40}\right)r_a\\
    &=\frac1{20}\cdot\frac{29}{40}\log6
      -\frac9{20}\cdot\frac9{40}\log9
      -\frac12\cdot\frac{11}{40}\log\frac{13}{10}\\
    &=\frac{29\log6-81\log9-110\log(13/10)}{800}<0.
\end{aligned}
\end{equation*}
The sign is strict because $\log9>\log6>0$ and $\log(13/10)>0$.
This construction is realizable by a tabular softmax policy: choose the logits to be $\log p$ at the ordinary context and $\log q$ at the privileged context.
All logits are finite, and the two context-specific tables can be chosen independently.

With $c=-\operatorname{Cov}_p(R,r)>0$, the reward expansion becomes
\begin{equation*}
    J(\theta';h)-J(\theta;h)=-c\eta+\mathcal O(\eta^2).
\end{equation*}
For all sufficiently small positive $\eta$, the remainder has magnitude at most $c\eta/2$.
The reward change is then at most $-c\eta/2<0$, even though the teacher has higher expected reward.
The matching-loss expansion also ensures that the same step strictly decreases the distillation objective.

Finally, extend the construction to any $n>3$ by keeping the first two actions and replacing the third by $n-2$ zero-reward actions.
For each $a=3,\ldots,n$, set
\begin{equation*}
    R(a)=0,\qquad
    p_a=\frac{1}{2(n-2)},\qquad
    q_a=\frac{13}{20(n-2)}.
\end{equation*}
These actions retain the original third action's total student and teacher probabilities, and each has the same log-ratio $\log(13/10)$.
Their combined contribution to the reward--signal covariance is therefore unchanged.
Both expected rewards are also unchanged, and full support is preserved.
The counterexample consequently holds for every $|\mathcal V|\geq3$.
\end{proof}

\subsection{Proof of Proposition~\ref{prop:CoCA-refinement}}
\label{app:proof-3}

\begin{cocarestated}{prop:CoCA-refinement}{Grounding-preserving resolution improvement}
For any $\beta\geq0$, CoCA preserves each criterion's budget and the sign of its token-level contributions, with
$\widehat{\mathbf C}_i(\beta)\in\mathcal C_i(\mathbf a_i)$.
For each selected criterion $j\in\mathcal S_i$ with $\Delta_{ij}>0$, the following holds whenever
$0<\beta<8\Delta_{ij}/b_{ij}^2$:
\begin{equation*}
    \textstyle
    \frac{1}{T_i}
    \sum_{t\notin\mathcal T_{ij}}
    \left|[\widehat{\mathbf C}_i(\beta)]_{j,t}\right|
    <
    |A_{ij}|
    \left(1-\frac{|\mathcal T_{ij}|}{T_i}\right).
    \tag{\ref{eq:CoCA-refinement}}
\end{equation*}
The right-hand side gives the same quantity under uniform allocation.
\end{cocarestated}

\begin{proof}
Fix a rollout $i$ and hold its criterion budgets, selected set, relevance scores, and relevant token sets fixed as $\beta$ varies.
We first prove budget and sign preservation, then compare the credit magnitude assigned to irrelevant positions with its value under uniform allocation.

\emph{Budget and sign preservation.}
Write $\omega_{ij,t}(\beta)=[\boldsymbol\omega_{ij}(\beta)]_t$.
For each selected criterion $j\in\mathcal S_i$, positivity of the exponential and normalization of the weights give
\begin{equation*}
    \omega_{ij,t}(\beta)
    =\frac{T_i e^{\beta d_{ij,t}}}{\sum_{u=1}^{T_i}e^{\beta d_{ij,u}}}>0,
    \qquad
    \frac1{T_i}\sum_{t=1}^{T_i}\omega_{ij,t}(\beta)=1.
\end{equation*}
For $j\notin\mathcal S_i$, the weights equal one, so positivity and normalization hold for every criterion.
Multiplication by the fixed budget then gives the following row-mean identity:
\begin{equation*}
    \frac1{T_i}\sum_{t=1}^{T_i}[\widehat{\mathbf C}_i(\beta)]_{j,t}
    =A_{ij}\frac1{T_i}\sum_{t=1}^{T_i}\omega_{ij,t}(\beta)
    =A_{ij}.
\end{equation*}
These row-mean identities show $\widehat{\mathbf C}_i(\beta)\in\mathcal C_i(\mathbf a_i)$.
Positive weights ensure that each nonzero row has its budget's sign, while zero budgets yield zero rows.
Taking absolute values rowwise and summing the signed credits over criteria gives
\begin{equation*}
    \frac1{T_i}\sum_{t=1}^{T_i}\left|[\widehat{\mathbf C}_i(\beta)]_{j,t}\right|=|A_{ij}|,
    \qquad
    \frac1{T_i}\sum_{t=1}^{T_i}\widehat A_{i,t}
    =\sum_{j=1}^{M}A_{ij}=A_i.
\end{equation*}
Thus, redistribution also preserves each row's total credit magnitude and the response-level mean advantage.
At $\beta=0$, all weights equal one, recovering uniform allocation.

\emph{Reduction to a comparison of exponential averages.}
Fix a selected criterion $j\in\mathcal S_i$ with $\Delta_{ij}>0$.
Because selection is restricted to nonzero budgets, $A_{ij}\neq0$.
For readability, write $T=T_i$, $A=A_{ij}$, $d_t=d_{ij,t}$, $\Delta=\Delta_{ij}$, and $b=b_{ij}$.
Let $\mathcal T=\mathcal T_{ij}$ and $\mathcal N=\{1,\ldots,T\}\setminus\mathcal T$.
The assumption $\varnothing\neq\mathcal T\subsetneq\{1,\ldots,T\}$ ensures
$m=|\mathcal T|>0$, $n=|\mathcal N|>0$, and $m+n=T$.
The relevant set is specified without using the scores, and its mean-score advantage is
\begin{equation*}
    \mu_{\mathcal T}=\frac1m\sum_{t\in\mathcal T}d_t,
    \qquad
    \mu_{\mathcal N}=\frac1n\sum_{t\in\mathcal N}d_t,
    \qquad
    \Delta=\mu_{\mathcal T}-\mu_{\mathcal N}>0.
\end{equation*}
Define the mean exponential scores on the two sets as
\begin{equation*}
    U(\beta)=\frac1m\sum_{t\in\mathcal T}e^{\beta d_t},
    \qquad
    V(\beta)=\frac1n\sum_{t\in\mathcal N}e^{\beta d_t}.
\end{equation*}
Substituting the normalized allocation weights gives the off-target credit magnitude
\begin{equation*}
    L(\beta)
    :=\frac1T\sum_{t\in\mathcal N}
      \left|[\widehat{\mathbf C}_i(\beta)]_{j,t}\right|
    =|A|\frac{\sum_{t\in\mathcal N}e^{\beta d_t}}
      {\sum_{u=1}^{T}e^{\beta d_u}}
      =|A|\frac{nV(\beta)}{mU(\beta)+nV(\beta)}.
\end{equation*}
Under uniform allocation, $U(0)=V(0)=1$, so $L(0)=|A|n/T$.
As $|A|$, $m$, $n$, and both exponential averages are positive,
$L(\beta)<L(0)$ is equivalent to $U(\beta)>V(\beta)$.
It remains to show that the positive mean-score gap guarantees this ordering under the stated restriction on $\beta$.

\emph{Controlling the effect of exponentiation.}
For the relevant positions, Jensen's inequality gives
\begin{equation*}
    U(\beta)=\frac1m\sum_{t\in\mathcal T}e^{\beta d_t}
    \geq\exp\!\left(\frac\beta m\sum_{t\in\mathcal T}d_t\right)
    =e^{\beta\mu_{\mathcal T}}.
\end{equation*}
For the irrelevant positions, we also need to control the increase in the exponential average caused by variation in the scores.
For $s\geq0$, define
\begin{equation*}
    \phi(s)=\log\!\left(\frac1n\sum_{t\in\mathcal N}e^{s d_t}\right),
    \qquad
    \nu_s(t)=\frac{e^{s d_t}}{\sum_{u\in\mathcal N}e^{s d_u}}.
\end{equation*}
Here $\nu_s$ is a probability distribution over the fixed set $\mathcal N$.
Differentiating the finite sum twice gives
\begin{equation*}
\begin{aligned}
\phi'(s)
    &=\frac{\sum_{t\in\mathcal N}d_t e^{s d_t}}
      {\sum_{t\in\mathcal N}e^{s d_t}}
      =\mathbb E_{t\sim\nu_s}[d_t],\\
\phi''(s)
    &=\frac{\sum_{t\in\mathcal N}d_t^2 e^{s d_t}}
      {\sum_{t\in\mathcal N}e^{s d_t}}
      -\left(\frac{\sum_{t\in\mathcal N}d_t e^{s d_t}}
      {\sum_{t\in\mathcal N}e^{s d_t}}\right)^2
=\operatorname{Var}_{t\sim\nu_s}(d_t).
\end{aligned}
\end{equation*}
To bound this variance uniformly in $s$, let
$d_{\min}=\min_t d_t$, $d_{\max}=\max_t d_t$, and
$c=(d_{\min}+d_{\max})/2$.
The score range is $b=d_{\max}-d_{\min}$, so $|d_t-c|\leq b/2$ for every position.
Therefore,
\begin{equation*}
    \phi''(s)
    =\mathbb E_{\nu_s}[(d_t-c)^2]
      -\bigl(\mathbb E_{\nu_s}[d_t]-c\bigr)^2
    \leq\mathbb E_{\nu_s}[(d_t-c)^2]\leq\frac{b^2}{4}.
\end{equation*}
At $s=0$, the distribution $\nu_0$ is uniform on $\mathcal N$, giving
$\phi(0)=0$ and $\phi'(0)=\mu_{\mathcal N}$.
Applying Taylor's formula with integral remainder and the uniform variance bound yields
\begin{equation*}
\begin{aligned}
\log V(\beta)=\phi(\beta)
    &=\beta\mu_{\mathcal N}
      +\int_0^\beta(\beta-s)\phi''(s)\,\mathrm ds\\
&\leq\beta\mu_{\mathcal N}
      +\frac{b^2}{4}\int_0^\beta(\beta-s)\,\mathrm ds
=\beta\mu_{\mathcal N}+\frac{\beta^2b^2}{8}.
\end{aligned}
\end{equation*}

\emph{Strict improvement over uniform allocation.}
Combining the upper bound on $\log V(\beta)$ with the lower bound
$\log U(\beta)\geq\beta\mu_{\mathcal T}$ yields
\begin{equation*}
    \log\frac{U(\beta)}{V(\beta)}
    \geq\beta(\mu_{\mathcal T}-\mu_{\mathcal N})-\frac{\beta^2b^2}{8}
    =\beta\left(\Delta-\frac{\beta b^2}{8}\right).
\end{equation*}
The condition $\Delta>0$ implies $b>0$: otherwise all scores would coincide and the mean gap would be zero.
For $0<\beta<8\Delta/b^2$, the right-hand side is strictly positive, and hence $U(\beta)>V(\beta)$.
This bound ensures that the mean-score advantage survives exponentiation.

Finally, substituting $U(\beta)>V(\beta)$ into the off-target magnitude gives
\begin{equation*}
    L(\beta)
    =|A|\frac{nV(\beta)}{mU(\beta)+nV(\beta)}
    <|A|\frac{nV(\beta)}{mV(\beta)+nV(\beta)}
      =|A|\frac nT
      =|A|\left(1-\frac mT\right).
\end{equation*}
Strictness uses $|A|>0$, $m,n>0$, and $V(\beta)>0$.
Restoring the indices proves the claimed reduction relative to uniform allocation.
Because the comparison uses absolute credit while every row retains its budget sign, it applies equally to positive and negative criterion budgets.
\end{proof}

\section{Additional Experimental Details}
\label{app:exp_detail}

\subsection{Datasets, Benchmarks, and Evaluation}
\label{app:eval}

\paragraph{Training datasets.}
We focus on medical and scientific QA tasks, training a separate policy for each domain on RaR-Medicine and RaR-Science~\citep{rar}.
Both datasets pair reasoning questions with instance-specific rubrics that identify essential, important, optional, and pitfall criteria.
The medical data cover diagnosis, treatment, and related medical reasoning; the science data cover scientific knowledge and multi-step problem solving.
We retain the criterion text and signed point values in the prepared data.
The training splits contain 17,926 medical questions and 18,333 scientific questions, with an average of 7.46 and 7.51 criteria per question, respectively.

\paragraph{Evaluation benchmarks.}
We evaluate each domain-specific policy on its corresponding RaR test split (2,242 questions for medicine and 2,292 for science) and on additional benchmarks that assess generalization across question sources and task requirements.

We use three additional medical benchmarks.
(1)~\textbf{HealthBench}~\citep{healthbench} provides 5,000 realistic medical conversations with physician-authored rubrics; we evaluate all final assistant responses in the context of preceding messages.
(2)~\textbf{\mbox{LLMEval-Med}}~\citep{llmevalmed} assesses clinical knowledge, understanding, reasoning, safety, and generation; we use its 667-question public test set.
(3)~\textbf{MedQA}~\citep{medqa} contains medical board-examination questions; we use its 1,273-question, four-option English test split.

The scientific suite includes two additional benchmarks.
(1)~\textbf{ResearchQA}~\citep{researchqa} evaluates answers to research questions using question-specific criteria derived from survey literature.
We use a prepared subset of 3,750 questions and preserve the original task instructions, including any citation, length, or temporal constraints.
(2)~\textbf{\mbox{GPQA-D}}~\citep{gpqa} is the Diamond subset of GPQA, comprising 198 expert-written multiple-choice questions in biology, chemistry, and physics.

\paragraph{Evaluation criteria.}
Each benchmark is evaluated using rubric criteria with associated point values.
RaR-Medicine, RaR-Science, and HealthBench assess response quality through weighted criteria, including negatively weighted penalties.
ResearchQA criteria describe the expected answer content and carry unit weights.
For LLMEval-Med, clinical usability is assessed through a single unit-weight criterion based on the expert reference answer and checklist.
The reported metric is the clinical-usability pass rate, with a passing standard corresponding to a rating of at least 4 out of 5.
MedQA and \mbox{GPQA-D} assess final-answer correctness through a single unit-weight criterion.
All criteria are assessed using the grading procedure described below.

\begin{table}[!t]
\centering
\caption{\small Training and evaluation configuration.}
\label{tab:appendix-training}
\small
\setlength{\tabcolsep}{5pt}
\renewcommand{\arraystretch}{1.08}
\begin{tabular}{@{}p{0.17\linewidth}p{\dimexpr0.83\linewidth-2\tabcolsep\relax}@{}}
\toprule
\textbf{Settings} & \textbf{Hyperparameters} \\
\midrule
\textbf{CoCA} & \raggedright
$k=2$, $\beta=5$ (medicine), $\beta=3$ (science) \tabularnewline
\midrule
\textbf{Sampling} & \raggedright
rollout\_n $=8$, top\_k $=-1$, top\_p $=1.0$, temperature $=1.0$\newline
max\_prompt\_length $=2{,}048$, max\_response\_length $=4{,}096$\newline
enable\_thinking $=\mathrm{False}$ \tabularnewline
\midrule
\textbf{Training} & \raggedright
ppo\_mini\_batch\_size $=32$, ppo\_micro\_batch\_size\_per\_gpu $=2$\newline
log\_prob\_micro\_batch\_size\_per\_gpu $=16$, train\_batch\_size $=64$\newline
AdamW, learning\_rate $=10^{-6}$, weight\_decay $=0.01$\newline
schedule $=\mathrm{constant}$, warmup ratio $=0$, grad\_clip $=1.0$\newline
ppo\_epochs $=1$, clip\_ratio $=0.2$, training\_steps $=400$\newline
entropy\_coeff $=0$, use\_kl\_loss $=\mathrm{False}$, use\_kl\_in\_reward $=\mathrm{False}$ \tabularnewline
\midrule
\textbf{Evaluation} & \raggedright
top\_p $=0.95$, temperature $=0.6$, 
max\_response\_length $=4{,}096$ \tabularnewline
\midrule
\textbf{Optimizations} & \raggedright
FSDP, vLLM (tensor parallel size $=2$), FlashAttention-2\newline
bf16 mixed precision, gradient checkpointing, padding removal\newline
Actor parameter and optimizer offload: disabled \tabularnewline
\bottomrule
\end{tabular}
\end{table}

\paragraph{LLM-based evaluation.}
We use gpt-oss-120b to evaluate all benchmarks and methods, with temperature $0$ and a 4,096-token output limit.
Given the task prompt, generated response, and rubric, the judge assesses all criteria in one call using the template in Figure~\ref{fig:appendix-judge-prompt}.
For each evaluation example $e$, let $s_{ej}$ denote the point value of criterion $j$ and $z_{ej}\in\{0,1\}$ its satisfaction indicator.
We compute HealthBench-style response and benchmark scores as
\begin{equation}
    P_e=\sum_{j:s_{ej}>0}s_{ej},
    \qquad
    r_e=\max\!\left\{0,\frac{\sum_j s_{ej}z_{ej}}{P_e}\right\},
    \qquad
    \operatorname{Score}(\mathcal D)=\frac{100}{|\mathcal D|}\sum_{e\in\mathcal D}r_e.
    \label{eq:appendix-eval-score}
\end{equation}
To reduce sampling variability, eight responses per question are generated for GPQA-D (avg@8) and two for LLMEval-Med (avg@2), with binary scores averaged across all questions and samples.
All other benchmarks use one response per question (avg@1).
Benchmark scores are reported as percentages, and the overall score is the unweighted mean of the seven benchmark scores.

\begin{figure}[!t]
\centering
\includegraphics[width=\linewidth]{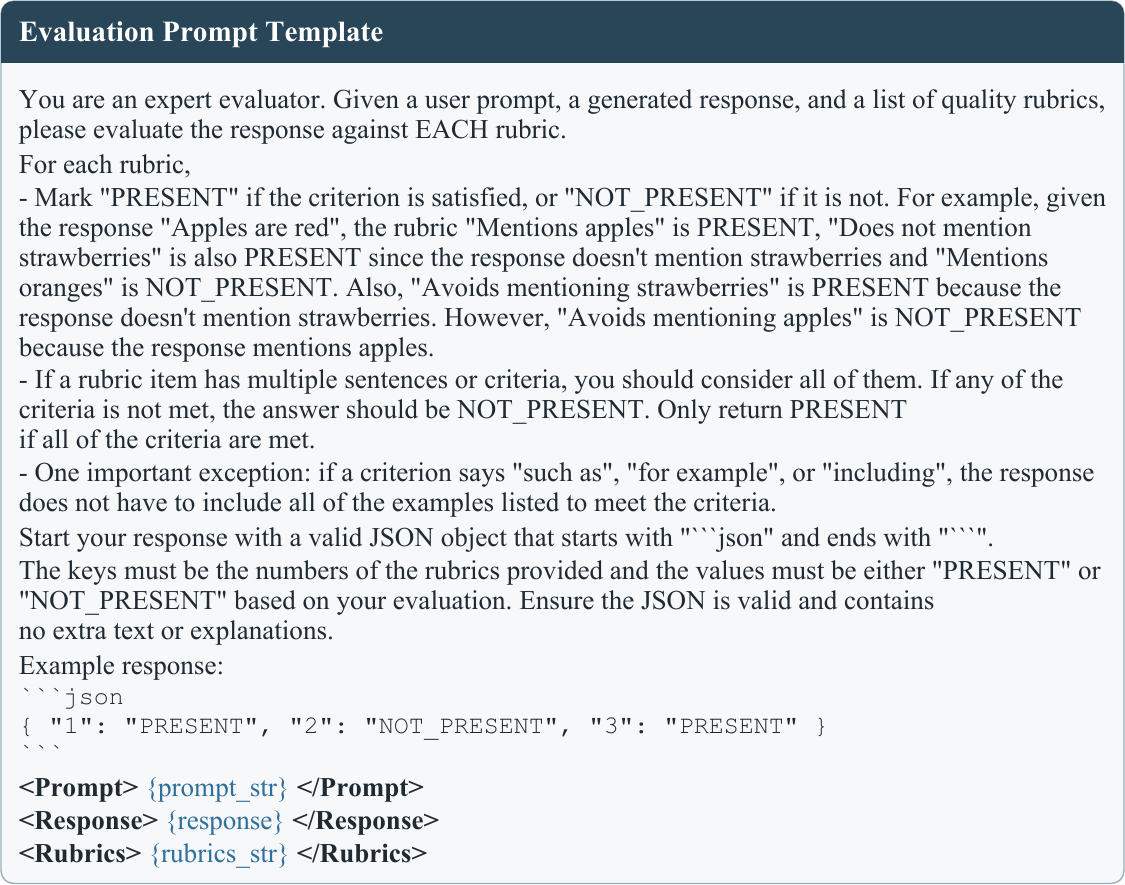}
\caption{\small Prompt template for the LLM judge. Blue placeholders denote the task prompt, model response, and numbered rubric criteria.}
\label{fig:appendix-judge-prompt}
\end{figure}

\subsection{Training and Implementation Details}
\label{app:train}

\paragraph{Optimization and decoding.}
We use \texttt{verl}~\citep{verl} with FSDP for policy training and vLLM for response generation.
Table~\ref{tab:appendix-training} summarizes the optimization and decoding settings.
We disable top-$k$ truncation during rollout sampling.
Evaluation responses are generated using only the task context, with rubric criteria reserved for the judge.

\paragraph{Criterion-conditioned inputs.}
For each selected criterion, we append the corresponding instruction in Figure~\ref{fig:appendix-criterion-prompt} to the final user message, separated from the original content by a blank line.
All preceding conversation messages are retained.
The template is selected according to the sign of the criterion's original point value $s_j$, independently of the sign of its group-relative credit budget $A_{ij}$.

\begin{figure}[!t]
\centering
\includegraphics[width=\linewidth]{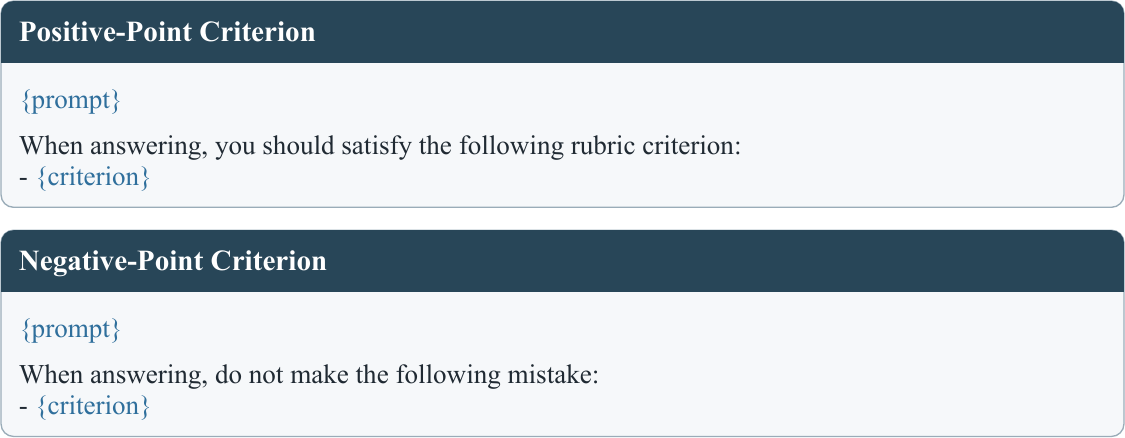}
\caption{\small Prompt templates for criterion-conditioned forward passes in CoCA. Blue placeholders denote the original final user message and the selected criterion.}
\label{fig:appendix-criterion-prompt}
\end{figure}

\paragraph{Response alignment and masking.}
We apply the model's chat template to each conditioned conversation and append the original sampled response token IDs.
This preserves the response sequence across scoring passes despite differences in context length.
We extract the logits that predict response tokens and align them by response position before computing relevance.
The response mask excludes prompt and padding positions from token-weight normalization and the policy loss.

\subsection{Baseline Reproduction}
\label{app:baseline}

We reproduce RaR~\citep{rar}, RuscaRL~\citep{ruscarl}, HeRL~\citep{herl}, POW3R~\citep{pow3r}, and CoRT~\citep{cort} following their official implementations or the standard procedures described in the original papers.

\paragraph{OPSD.}
We implement rubric-conditioned OPSD using the self-teacher scoring and distillation components of the official SDPO codebase~\citep{sdpo}.
For each question, the privileged information consists of the complete set of rubric criteria.
The student generates a response from the ordinary task prompt, and the self-teacher scores the same response under a context augmented with these criteria.
Teacher predictions are conditioned on the sampled response prefix and detached from backpropagation.
OPSD alone optimizes the reverse-KL objective in Eq.~(\ref{eq:opsd}) over response tokens.

\paragraph{RaR+OPSD.}
We combine the RaR policy loss with the OPSD loss:
\begin{equation}
    \mathcal L_{\mathrm{RaR+OPSD}}
    =\mathcal L_{\mathrm{GRPO}}+\lambda\mathcal L_{\mathrm{OPSD}},
    \qquad \lambda=0.01.
    \label{eq:appendix-hybrid-loss}
\end{equation}
Here, $\mathcal L_{\mathrm{GRPO}}$ is the policy loss in Eq.~(\ref{eq:grpo}), computed using response-level advantages derived from rubric rewards, and $\mathcal L_{\mathrm{OPSD}}$ is the teacher-matching loss in Eq.~(\ref{eq:opsd}).
Both terms are computed on the same sampled responses and jointly optimized in each policy update.
The student receives the ordinary task prompt, with rubric information confined to the teacher context.
We select $\lambda=0.01$ according to the observed gradient magnitudes: during training, the gradient norm of the weighted distillation term $\lambda\mathcal L_{\mathrm{OPSD}}$ is approximately one fifth of that of the policy loss $\mathcal L_{\mathrm{GRPO}}$.
This keeps distillation active as an auxiliary objective while retaining policy optimization as the primary training signal.
An excessively large $\lambda$ can let distillation dominate the combined objective, whereas an excessively small one can make its contribution negligible.

\section{Additional Experimental Results}
\label{app:exp_result}

\subsection{Results Across Base Models and Judges}
\label{app:model}
\paragraph{Consistent gains across base models.}
To examine whether our findings extend beyond Qwen3-8B, we evaluate CoCA on Qwen3-4B~\citep{qwen3} and Olmo3-7B-Instruct-SFT~\citep{olmo3} against three representative baselines: RaR, RaR+OPSD, and HeRL, using gpt-oss-120b as the judge.
As shown in Table~\ref{tb:main_add}, CoCA achieves average scores of $58.0\%$ and $41.5\%$, exceeding RaR by $4.1\%$ and $1.3\%$, respectively.
CoCA also surpasses the strongest baseline HeRL, by $2.4\%$ and $1.4\%$ on average.
This consistent advantage supports the benefit of resolving criterion-level credit beyond response-level rubric rewards.
Meanwhile, RaR+OPSD underperforms RaR on all seven benchmarks for both models, with average declines of $1.2\%$ and $2.1\%$, reinforcing our finding that an auxiliary rubric-conditioned distillation objective can undermine reward-driven learning.
Together, these results reproduce the main-table trends across model scales and architectures, supporting the effectiveness of outcome-grounded criterion-level credit assignment beyond the default base model.

\begin{table*}[!t]
\caption{\small Experimental results (\%) for Qwen3-4B and Olmo3-7B-Instruct-SFT with gpt-oss-120b as the judge.
Avg.\ denotes the mean across all seven benchmarks.
The best results within each model block are highlighted in \textbf{bold}.
Arrows indicate \textcolor{OliveGreen}{improvements ($\uparrow$)} or \textcolor{Maroon}{declines ($\downarrow$)} relative to the corresponding initial model.}
\centering
\renewcommand{\arraystretch}{1.3}
\setlength{\tabcolsep}{4pt}
\resizebox{\textwidth}{!}{
\begin{tabular}{l|c@{\hspace{5pt}}c@{\hspace{5pt}}c@{\hspace{4pt}}>{\columncolor{white}[4pt][8pt]}c@{\hspace{12pt}}c@{\hspace{6pt}}>{\columncolor{white}[4pt][6pt]}c@{\hspace{10pt}}>{\columncolor{white}[4pt][7pt]}c@{\hspace{7pt}}|@{\hspace{7pt}}>{\columncolor{white}[7pt][4pt]}c}
\toprule
\multirow{2}{*}{\textbf{Method}}
& \multicolumn{4}{c@{\hspace{12pt}}}{\textbf{Medicine}}
& \multicolumn{3}{@{}c@{\hspace{7pt}}|@{\hspace{7pt}}}{\textbf{Science}}
& \multirow{2}{*}{\textbf{Avg.}} \\
\cmidrule(l{4pt}r{12pt}){2-5}\cmidrule(l{0pt}r{14.4pt}){6-8}
& \textbf{RaR-Medicine}
& \textbf{HealthBench}
& \textbf{LLMEval-Med}
& \hspace*{-1pt}\textbf{MedQA}\hspace*{1pt}
& \textbf{RaR-Science}
& \textbf{ResearchQA}
& \hspace*{-2pt}\textbf{GPQA-D}\hspace*{2pt}
& \\
\midrule
\multicolumn{9}{c}{Qwen3-4B} \\
\midrule
Initial & 45.6 \basex{0.0} & 37.8 \basex{0.0} & 48.3 \basex{0.0} & 66.0 \basex{0.0} & 64.3 \basex{0.0} & 43.8 \basex{0.0} & 43.2 \basex{0.0} & 49.9 \basex{0.0} \\
RaR & 55.2 \uprise{9.6} & 43.6 \uprise{5.8} & 50.3 \uprise{2.0} & 66.2 \uprise{0.2} & 69.2 \uprise{4.9} & 48.3 \uprise{4.5} & 44.6 \uprise{1.4} & 53.9 \uprise{4.0} \\
RaR+OPSD & 52.4 \uprise{6.8} & 42.2 \uprise{4.4} & 49.1 \uprise{0.8} & 65.8 \down{0.2} & 68.0 \uprise{3.7} & 47.6 \uprise{3.8} & 43.7 \uprise{0.5} & 52.7 \uprise{2.8} \\
HeRL & 58.0 \uprise{12.4} & 45.7 \uprise{7.9} & 51.9 \uprise{3.6} & 68.7 \uprise{2.7} & 70.3 \uprise{6.0} & 49.0 \uprise{5.2} & 45.4 \uprise{2.2} & 55.6 \uprise{5.7} \\
\rowcolor{gray!15}
\textbf{CoCA (ours)} & \textbf{61.9} \uprise{16.3} & \textbf{49.1} \uprise{11.3} & \textbf{54.0} \uprise{5.7} & \textbf{70.6} \uprise{4.6} & \textbf{72.1} \uprise{7.8} & \textbf{50.7} \uprise{6.9} & \textbf{47.9} \uprise{4.7} & \textbf{58.0} \uprise{8.1} \\
\midrule
\multicolumn{9}{c}{Olmo3-7B-Instruct-SFT} \\
\midrule
Initial & 31.8 \basex{0.0} & 26.1 \basex{0.0} & 12.0 \basex{0.0} & 49.3 \basex{0.0} & 52.6 \basex{0.0} & 39.3 \basex{0.0} & 32.2 \basex{0.0} & 34.8 \basex{0.0} \\
RaR & 43.0 \uprise{11.2} & 32.8 \uprise{6.7} & 13.5 \uprise{1.5} & 51.8 \uprise{2.5} & 62.0 \uprise{9.4} & 45.8 \uprise{6.5} & 32.6 \uprise{0.4} & 40.2 \uprise{5.4} \\
RaR+OPSD & 40.1 \uprise{8.3} & 31.0 \uprise{4.9} & 11.3 \down{0.7} & 49.7 \uprise{0.4} & 61.3 \uprise{8.7} & 44.8 \uprise{5.5} & 28.6 \down{3.6} & 38.1 \uprise{3.3} \\
HeRL & 42.7 \uprise{10.9} & 32.4 \uprise{6.3} & \textbf{15.8} \uprise{3.8} & 50.9 \uprise{1.6} & 62.6 \uprise{10.0} & 45.7 \uprise{6.4} & 30.7 \down{1.5} & 40.1 \uprise{5.3} \\
\rowcolor{gray!15}
\textbf{CoCA (ours)} & \textbf{44.1} \uprise{12.3} & \textbf{33.9} \uprise{7.8} & 15.5 \uprise{3.5} & \textbf{53.5} \uprise{4.2} & \textbf{63.1} \uprise{10.5} & \textbf{47.1} \uprise{7.8} & \textbf{33.4} \uprise{1.2} & \textbf{41.5} \uprise{6.7} \\
\bottomrule
\end{tabular}
}
\label{tb:main_add}
\end{table*}

\paragraph{Consistent gains across judge models.}
We extend the evaluation to DeepSeek-V4.1-Flash~\citep{deepseek4.1} and rescore the same set of responses from all Qwen3-8B-based methods, keeping all other evaluation settings unchanged.
As shown in Table~\ref{tb:main_ds}, CoCA outperforms all baselines on all seven benchmarks, achieving an average score of $64.8\%$.
It improves over RaR, the standard rubric-based RL baseline, by an absolute $3.4\%$, and surpasses HeRL, the strongest baseline on average, by $2.5\%$.
These gains are comparable to the $4.2\%$ and $3.1\%$ margins under gpt-oss-120b.
Although the absolute scores vary across judges, the relative performance gaps among methods remain broadly consistent.
This consistency supports the robustness of CoCA's advantage across judge models.

\begin{table*}[!t]
\caption{\small Experimental results (\%) on medical and scientific benchmarks, using Qwen3-8B as the base model and DeepSeek-V4.1-Flash as the judge. Avg.\ denotes the mean score across all seven benchmarks, and the best results are highlighted in \textbf{bold}. Arrows indicate \textcolor{OliveGreen}{improvements ($\uparrow$)} or \textcolor{Maroon}{declines ($\downarrow$)} relative to the initial model.}
\centering
\renewcommand{\arraystretch}{1.3}
\setlength{\tabcolsep}{4pt}
\resizebox{\textwidth}{!}{
\begin{tabular}{l|c@{\hspace{5pt}}c@{\hspace{5pt}}c@{\hspace{4pt}}>{\columncolor{white}[4pt][8pt]}c@{\hspace{12pt}}c@{\hspace{6pt}}>{\columncolor{white}[4pt][6pt]}c@{\hspace{10pt}}>{\columncolor{white}[4pt][7pt]}c@{\hspace{7pt}}|@{\hspace{7pt}}>{\columncolor{white}[7pt][4pt]}c}
\toprule
\multirow{2}{*}{\textbf{Method}}
& \multicolumn{4}{c@{\hspace{12pt}}}{\textbf{Medicine}}
& \multicolumn{3}{@{}c@{\hspace{7pt}}|@{\hspace{7pt}}}{\textbf{Science}}
& \multirow{2}{*}{\textbf{Avg.}} \\
\cmidrule(l{4pt}r{12pt}){2-5}\cmidrule(l{0pt}r{14.4pt}){6-8}
& \textbf{RaR-Medicine}
& \textbf{HealthBench}
& \textbf{LLMEval-Med}
& \hspace*{-1pt}\textbf{MedQA}\hspace*{1pt}
& \textbf{RaR-Science}
& \textbf{ResearchQA}
& \hspace*{-2pt}\textbf{GPQA-D}\hspace*{2pt}
& \\
\midrule
Initial
& 57.9 \basex{0.0}
& 52.1 \basex{0.0}
& 44.3 \basex{0.0}
& 72.0 \basex{0.0}
& 67.7 \basex{0.0}
& 51.1 \basex{0.0}
& 49.7 \basex{0.0}
& 56.4 \basex{0.0} \\
RaR
& 68.2 \uprise{10.3}
& 57.5 \uprise{5.4}
& 47.4 \uprise{3.1}
& 73.8 \uprise{1.8}
& 73.1 \uprise{5.4}
& 58.0 \uprise{6.9}
& 51.6 \uprise{1.9}
& 61.4 \uprise{5.0} \\
OPSD
& 48.9 \down{9.0}
& 38.2 \down{13.9}
& 42.9 \down{1.4}
& 69.0 \down{3.0}
& 66.1 \down{1.6}
& 48.7 \down{2.4}
& 46.6 \down{3.1}
& 51.5 \down{4.9} \\
RaR+OPSD
& 64.6 \uprise{6.7}
& 52.3 \uprise{0.2}
& 44.6 \uprise{0.3}
& 72.1 \uprise{0.1}
& 68.3 \uprise{0.6}
& 51.2 \uprise{0.1}
& 48.8 \down{0.9}
& 57.4 \uprise{1.0} \\
RuscaRL
& 60.1 \uprise{2.2}
& 54.5 \uprise{2.4}
& 44.9 \uprise{0.6}
& 73.8 \uprise{1.8}
& 66.4 \down{1.3}
& 54.1 \uprise{3.0}
& 44.4 \down{5.3}
& 56.9 \uprise{0.5} \\
HeRL
& 70.6 \uprise{12.7}
& 60.3 \uprise{8.2}
& 48.6 \uprise{4.3}
& 73.3 \uprise{1.3}
& 72.6 \uprise{4.9}
& 58.9 \uprise{7.8}
& 51.6 \uprise{1.9}
& 62.3 \uprise{5.9} \\
POW3R
& 69.6 \uprise{11.7}
& 58.5 \uprise{6.4}
& 48.1 \uprise{3.8}
& 72.4 \uprise{0.4}
& 72.3 \uprise{4.6}
& 59.7 \uprise{8.6}
& 52.0 \uprise{2.3}
& 61.8 \uprise{5.4} \\
CoRT
& 68.9 \uprise{11.0}
& 57.8 \uprise{5.7}
& 47.8 \uprise{3.5}
& 75.9 \uprise{3.9}
& 71.6 \uprise{3.9}
& 58.4 \uprise{7.3}
& 47.7 \down{2.0}
& 61.2 \uprise{4.8} \\
\rowcolor{gray!15}
\textbf{CoCA (ours)}
& \textbf{73.1} \uprise{15.2}
& \textbf{63.3} \uprise{11.2}
& \textbf{51.1} \uprise{6.8}
& \textbf{77.6} \uprise{5.6}
& \textbf{74.2} \uprise{6.5}
& \textbf{60.0} \uprise{8.9}
& \textbf{54.4} \uprise{4.7}
& \textbf{64.8} \uprise{8.4} \\
\bottomrule
\end{tabular}
}
\label{tb:main_ds}
\end{table*}

\subsection{Quantitative Validation of Criterion--Token Relevance}
\label{app:relevance-validation}

\paragraph{Annotation protocol.}
To test whether CoCA's criterion-induced distribution shifts localize the parts of a response relevant to each rubric criterion, we randomly sample 500 criterion--response pairs from each domain (Medicine and Science).
For each pair, GPT-6 Astra annotates words that directly express criterion-relevant content or functionally support it, such as a heading or transition introducing that content.
We then manually review 100 pairs per domain ($20\%$) to check the quality of these annotations.
For analysis, we map the word-level spans to the corresponding response-token positions.
Figure~\ref{fig:annotation} presents one complete response from each domain.
Dark blue marks annotated words and pale blue marks the remaining words.
We refer to the unmarked positions in the eligible response-token set as \emph{remaining tokens}.

\begin{figure}[!t]
\centering
\includegraphics[width=\linewidth]{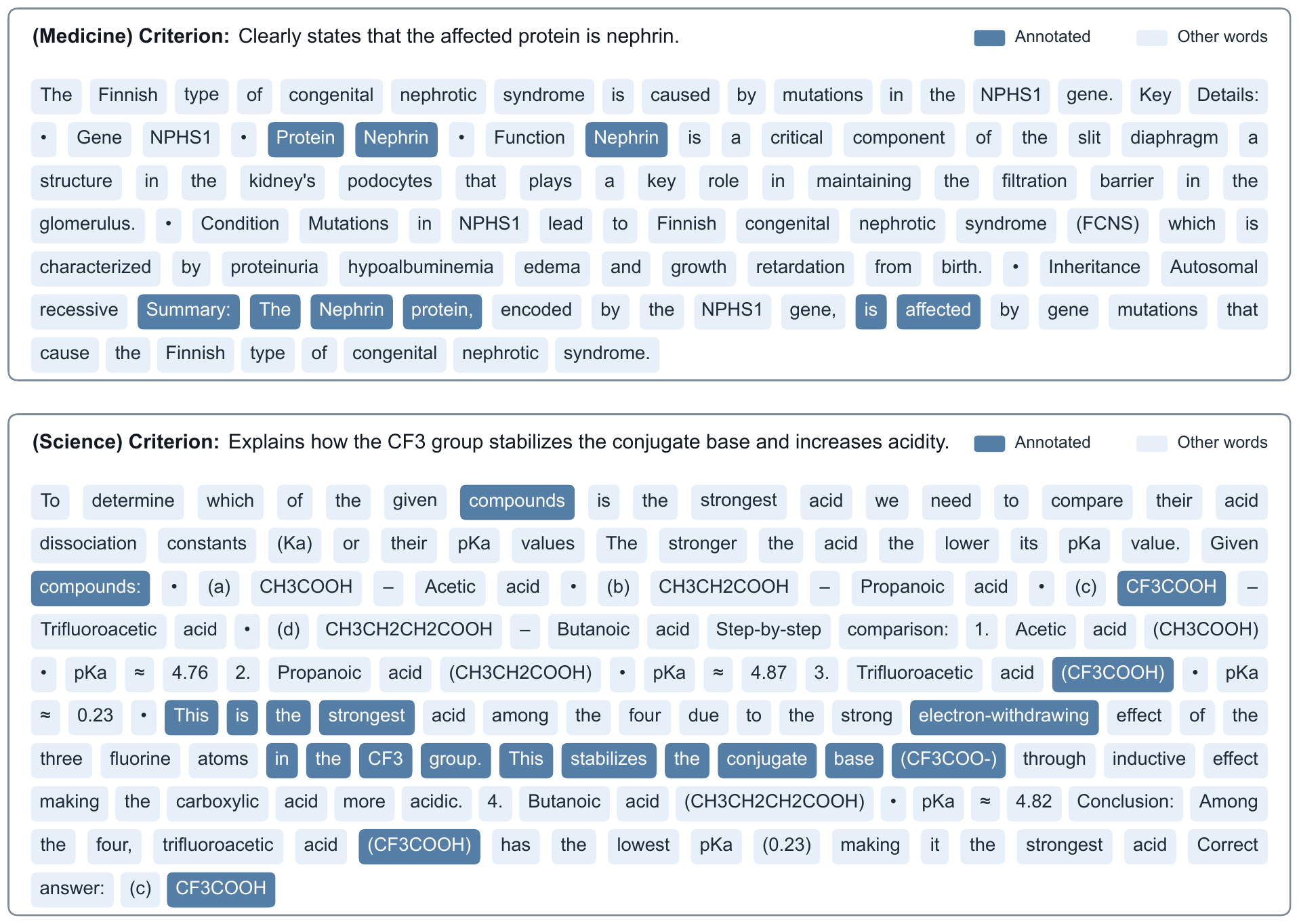}
\caption{\small Complete word-level annotation examples from Medicine (top) and Science (bottom), with the corresponding rubric criterion shown above each response.
Dark and pale blue denote annotated and remaining words, respectively.
Subword pieces are merged into words for display.}
\label{fig:annotation}
\end{figure}

\paragraph{Pair-level separation.}
For each criterion--response pair $p$, we compute the mean JSD over annotated positions $\mathcal T_p$ and remaining eligible positions $\mathcal N_p$, together with their enrichment ratio:
\begin{equation}
    \mu_p^{\mathrm{ann}}
    = \frac{1}{|\mathcal T_p|}\sum_{t\in\mathcal T_p} d_{p,t},
    \qquad
    \mu_p^{\mathrm{rem}}
    = \frac{1}{|\mathcal N_p|}\sum_{t\in\mathcal N_p} d_{p,t},
    \qquad
    r_p
    = \frac{\mu_p^{\mathrm{ann}}}{\mu_p^{\mathrm{rem}}}.
\label{eq:relevance-validation-metrics}
\end{equation}
Here, $d_{p,t}$ denotes the JSD relevance score at position $t$.
We report the \emph{positive-gap rate}, defined as the percentage of pairs satisfying $\mu_p^{\mathrm{ann}}>\mu_p^{\mathrm{rem}}$.
For each threshold $\tau$, we also report the percentage of pairs with an enrichment ratio $r_p\geq\tau$ to quantify the magnitude of this separation.
As shown in Figure~\ref{fig:relevance}, the positive-gap rate is $97.6\%$ in Medicine and $100.0\%$ in Science.
The fractions with $r_p\geq1.5$ are $71.0\%$ and $81.2\%$, those with $r_p\geq2.0$ are $47.0\%$ and $61.6\%$, and those with $r_p\geq3.0$ are $19.2\%$ and $31.6\%$, respectively.
Overall, $988$ of the $1{,}000$ pairs ($98.8\%$) satisfy $\Delta_p=\mu_p^{\mathrm{ann}}-\mu_p^{\mathrm{rem}}>0$.
These results show that JSD is higher on average at criterion-relevant positions, supporting its use for criterion-level credit allocation.
The observed separation also provides empirical support for the positive mean-gap condition underlying the credit-resolution guarantee in Proposition~\ref{prop:CoCA-refinement}.

\subsection{Analysis of Training Efficiency}
\label{app:efficiency}

\paragraph{Performance versus training time.}
We evaluate Qwen3-8B trained with CoCA and RaR on the Medicine validation set every 25 steps over 400 training steps, using the same $8\times$A100 GPUs.
Given the differences in per-step cost across methods, we compare validation performance as a function of training time rather than training steps.
As shown in the left panel of Figure~\ref{fig:efficiency}, CoCA reaches $55.57\%$ after $38.77$ hours, compared with $47.83\%$ after $37.35$ hours for RaR.
This corresponds to a gain of $7.74$ percentage points with only $3.8\%$ more training time.
Moreover, CoCA achieves $49.22\%$ at step 125 after only $12.12$ hours, already surpassing RaR’s best performance at step 375 ($35.02$ hours). Thus, CoCA achieves higher performance with a $2.89\times$ speedup in training time.
These results demonstrate that the learning gains from CoCA's additional forward passes outweigh their computational overhead, yielding a net improvement in training efficiency.

\begin{figure}[!t]
\centering
\includegraphics[width=\linewidth]{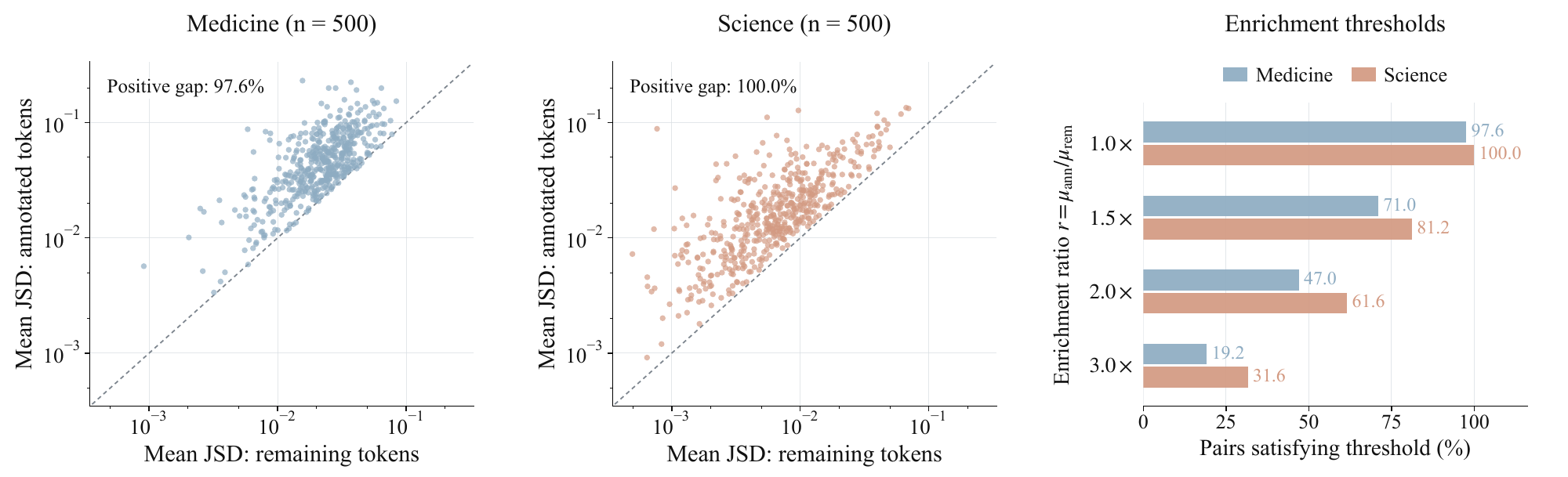}
\caption{\small Criterion--token relevance over 1,000 annotated pairs.
The scatterplots compare mean JSD at annotated and remaining positions; the dashed diagonal denotes equality.
The right panel reports the percentage of pairs meeting each enrichment-ratio threshold.}
\label{fig:relevance}
\end{figure}

\begin{figure}[!t]
\centering
\includegraphics[width=\linewidth]{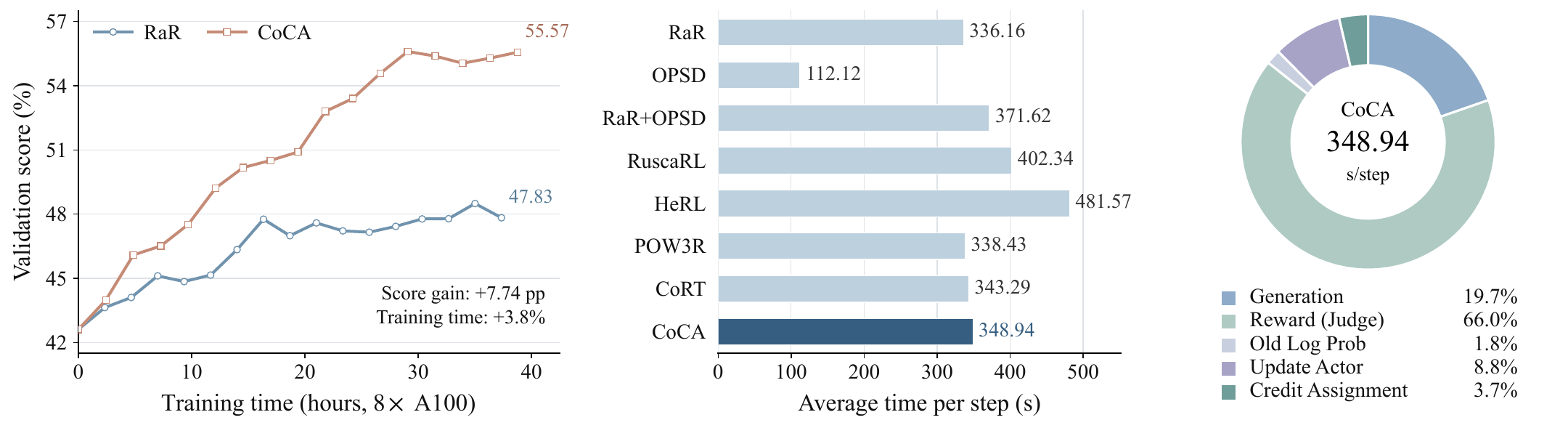}
\caption{\small Training efficiency of Qwen3-8B on Medicine using $8\times$A100 GPUs.
Left: validation performance versus training time.
Middle: average training step time for all methods.
Right: CoCA's runtime breakdown, with the total step time shown at the center and component proportions listed below.}
\label{fig:efficiency}
\vspace{-4pt}
\end{figure}

\paragraph{Per-step cost and runtime breakdown.}
To quantify CoCA's additional training cost, we measure the average step time of all eight methods under the same model, domain, and hardware configuration and profile its runtime.
As shown in the middle panel of Figure~\ref{fig:efficiency}, CoCA takes $348.94$ seconds per step compared with $336.16$ seconds for RaR, an increase of only $3.8\%$ despite its additional criterion-conditioned forward passes.
Its step time is also lower than those of RaR+OPSD, RuscaRL, and HeRL, showing that its finer-grained credit assignment remains computationally competitive with other approaches to rubric supervision.
The right panel shows that credit assignment takes only $12.78$ seconds per step ($3.7\%$ of the total), while judge-based reward evaluation and response generation together account for $85.7\%$.
Most training time is therefore spent producing responses and obtaining rubric feedback, with comparatively little devoted to allocating that feedback across tokens.
This is consistent with CoCA's design: it reuses criterion-level judgments from reward evaluation and restricts additional criterion-conditioned relevance scoring to selected criteria.
Together with the gains in learning efficiency, these results indicate that CoCA makes more effective use of costly rubric feedback with only modest additional computation.

\subsection{Analysis of Training Dynamics}
\label{app:dynamics}

To better understand how criterion-level credit assignment affects policy optimization, we compare the training dynamics of CoCA and RaR in Medicine and Science.
Specifically, we track training reward, response length, and policy entropy over $400$ training steps to characterize learning progress, generation behavior, and exploration, as shown in Figure~\ref{fig:dynamics}.

\begin{table*}[!t]
\caption{\small Ablation results (\%) for criterion selection strategies.
Arrows indicate changes relative to CoCA.}
\centering
\renewcommand{\arraystretch}{1.3}
\setlength{\tabcolsep}{4pt}
\resizebox{\textwidth}{!}{
\begin{tabular}{l@{\hspace{7pt}}|c@{\hspace{5pt}}c@{\hspace{5pt}}c@{\hspace{4pt}}>{\columncolor{white}[4pt][8pt]}c@{\hspace{12pt}}c@{\hspace{6pt}}>{\columncolor{white}[4pt][6pt]}c@{\hspace{10pt}}>{\columncolor{white}[4pt][7pt]}c@{\hspace{7pt}}|@{\hspace{7pt}}>{\columncolor{white}[7pt][4pt]}c}
\toprule
\multirow{2}{*}{\textbf{Strategy}}
& \multicolumn{4}{c@{\hspace{12pt}}}{\textbf{Medicine}}
& \multicolumn{3}{@{}c@{\hspace{7pt}}|@{\hspace{7pt}}}{\textbf{Science}}
& \multirow{2}{*}{\textbf{Avg.}} \\
\cmidrule(l{4pt}r{12pt}){2-5}
\cmidrule(l{0pt}r{14.4pt}){6-8}
& \textbf{RaR-Medicine}
& \textbf{HealthBench}
& \textbf{LLMEval-Med}
& \hspace*{-1pt}\textbf{MedQA}\hspace*{1pt}
& \textbf{RaR-Science}
& \textbf{ResearchQA}
& \hspace*{-2pt}\textbf{GPQA-D}\hspace*{2pt}
& \\
\midrule
\textbf{CoCA}
& \textbf{69.8} \basex{0.0}
& \textbf{54.3} \basex{0.0}
& \textbf{61.8} \basex{0.0}
& \textbf{77.6} \basex{0.0}
& \textbf{74.2} \basex{0.0}
& \textbf{59.6} \basex{0.0}
& \textbf{54.4} \basex{0.0}
& \textbf{64.5} \basex{0.0} \\
\midrule
Random
& 67.3 \down{2.5}
& 49.6 \down{4.7}
& 56.1 \down{5.7}
& 74.1 \down{3.5}
& 73.4 \down{0.8}
& 57.9 \down{1.7}
& 51.3 \down{3.1}
& 61.4 \down{3.1} \\
Top $|s_j|$
& 69.4 \down{0.4}
& 52.3 \down{2.0}
& 59.0 \down{2.8}
& 76.8 \down{0.8}
& 74.0 \down{0.2}
& 58.6 \down{1.0}
& 52.2 \down{2.2}
& 63.2 \down{1.3} \\
\bottomrule
\end{tabular}
}
\label{tb:ablation_select}
\end{table*}

\begin{figure}[!t]
\centering
\includegraphics[width=\linewidth]{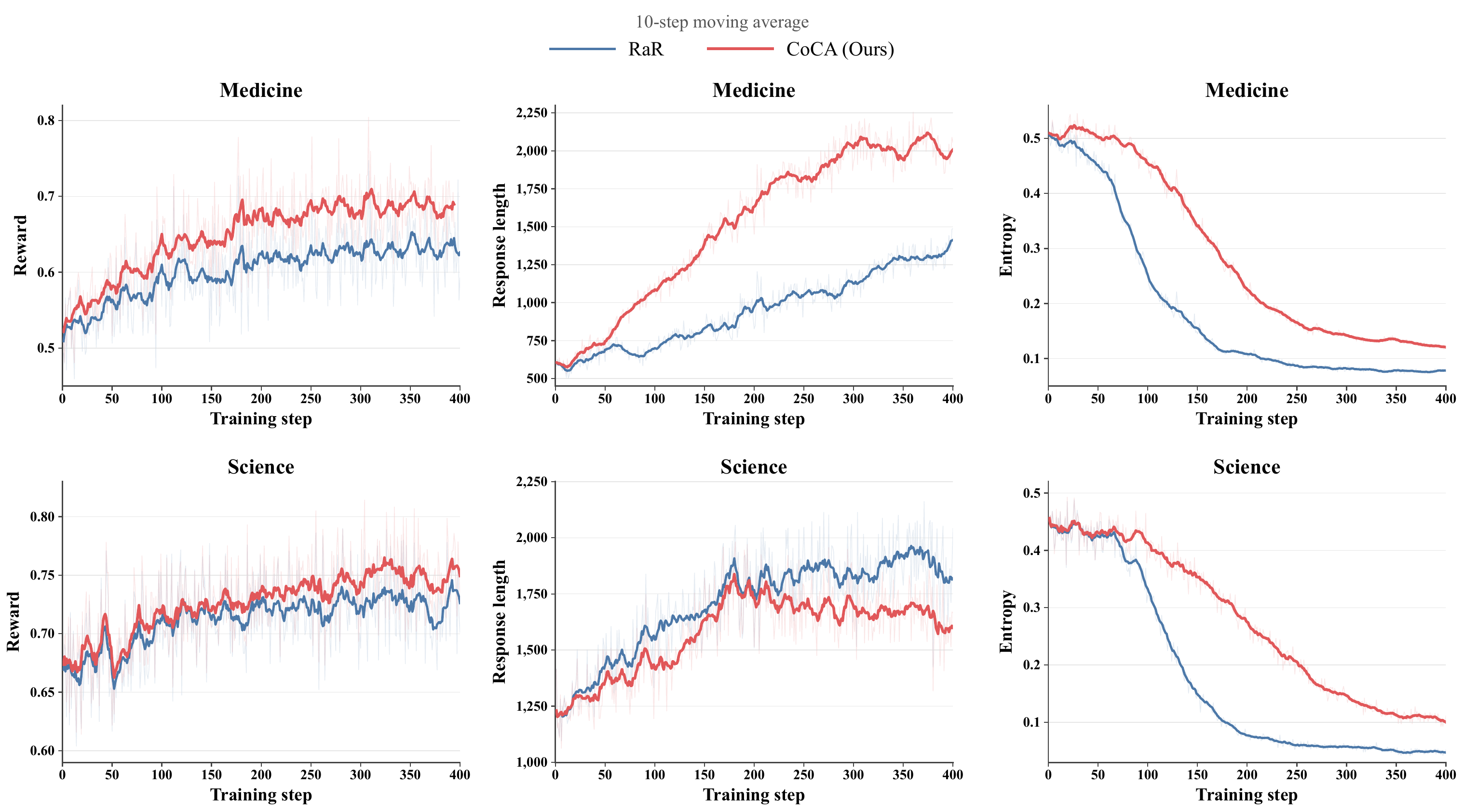}
\caption{\small Training dynamics of CoCA and RaR in Medicine (top) and Science (bottom).
Columns show reward, response length, and policy entropy, with faint stepwise curves and bold $10$-step moving averages.}
\label{fig:dynamics}
\end{figure}

\paragraph{Reward optimization.}
We first examine reward progression to assess the effectiveness of rubric-based optimization.
In Medicine, CoCA establishes a reward advantage early in training and maintains a clear gap as both methods approach a plateau.
In Science, the two methods exhibit similar initial progress, but CoCA gradually achieves higher rewards, with the difference becoming more apparent in the later stages.
This advantage persists through the end of training in both domains, suggesting that criterion-level credit assignment helps the policy make more effective use of rubric feedback and reach higher reward levels.

\paragraph{Response length.}
We additionally compare response length to assess whether the reward gains depend on producing longer outputs.
In Medicine, CoCA's responses become substantially longer than RaR's as training progresses, alongside its higher reward.
However, in Science, CoCA's response length peaks around the middle of training and then decreases, yielding shorter responses than RaR in the second half even as its reward improves further.
These contrasting trends show that CoCA can achieve higher reward with either longer or shorter responses, depending on the domain, and that its gains do not require uniformly increasing output length.

\paragraph{Exploration during training.}
Finally, we track policy entropy to characterize how exploration evolves during optimization.
In both domains, RaR exhibits a rapid entropy decline early in training, followed by a prolonged period at a low level.
CoCA's entropy decreases more gradually and remains higher through the later stages, indicating slower concentration of its token distributions.
Together with the higher rewards, this pattern suggests that CoCA supports effective reward optimization while retaining more scope for exploration.

\subsection{Effect of Criterion Selection}
\label{app:criterion-selection}

To assess the effect of criterion selection, we compare CoCA's default strategy with two alternatives, keeping the selection budget at $k=2$ and other training settings unchanged.
CoCA selects criteria with the largest absolute credits $|A_{ij}|$, whereas \textit{Random} selects criteria randomly and \textit{Top $|s_j|$} selects those with the largest absolute rubric point values.
Unselected criteria retain uniform allocation.

As shown in Table~\ref{tb:ablation_select}, \textit{Random} and \textit{Top $|s_j|$} achieve average scores of $61.4\%$ and $63.2\%$, respectively, falling below CoCA by $3.1\%$ and $1.3\%$.
Both alternatives underperform CoCA on all seven benchmarks, while \textit{Top $|s_j|$} consistently outperforms \textit{Random}.
The advantage of \textit{Top $|s_j|$} over \textit{Random} suggests that rubric point values provide a useful basis for prioritization.
CoCA further incorporates the magnitude of each criterion outcome's deviation from the rollout-group average through $|A_{ij}|$.
Its additional gains support directing the limited criterion-conditioned computation toward criteria with larger response-specific credit magnitudes.

\begin{table*}[!t]
\caption{\small Standard deviations ($\%$) of benchmark scores across three random seeds.
Avg.\ reports the standard deviation of the per-seed average score over all seven benchmarks.}
\centering
\renewcommand{\arraystretch}{1.3}
\setlength{\tabcolsep}{4pt}
\resizebox{\textwidth}{!}{
\begin{tabular}{l@{\hspace{7pt}}|c@{\hspace{5pt}}c@{\hspace{5pt}}c@{\hspace{4pt}}>{\columncolor{white}[4pt][8pt]}c@{\hspace{12pt}}c@{\hspace{6pt}}>{\columncolor{white}[4pt][6pt]}c@{\hspace{10pt}}>{\columncolor{white}[4pt][7pt]}c@{\hspace{7pt}}|@{\hspace{7pt}}>{\columncolor{white}[7pt][4pt]}c}
\toprule
\multirow{2}{*}{\textbf{Method}}
& \multicolumn{4}{c@{\hspace{12pt}}}{\textbf{Medicine}}
& \multicolumn{3}{@{}c@{\hspace{7pt}}|@{\hspace{7pt}}}{\textbf{Science}}
& \multirow{2}{*}{\textbf{Avg.}} \\
\cmidrule(l{4pt}r{12pt}){2-5}
\cmidrule(l{0pt}r{14.4pt}){6-8}
& \textbf{RaR-Medicine}
& \textbf{HealthBench}
& \textbf{LLMEval-Med}
& \hspace*{-1pt}\textbf{MedQA}\hspace*{1pt}
& \textbf{RaR-Science}
& \textbf{ResearchQA}
& \hspace*{-2pt}\textbf{GPQA-D}\hspace*{2pt}
& \\
\midrule
RaR
& 0.26 & 0.25 & 0.44 & 0.51 & 0.18 & 0.39 & 0.53 & 0.24 \\
\textbf{CoCA}
& 0.31 & 0.22 & 0.37 & 0.45 & 0.29 & 0.30 & 0.48 & 0.17 \\
\bottomrule
\end{tabular}
}
\label{tb:multi_seed}
\end{table*}

\subsection{Robustness Across Random Seeds}
\label{app:multi-seed}
To assess sensitivity to training randomness, we repeat the CoCA and RaR experiments with two additional random seeds, keeping all training and evaluation settings unchanged.
Table~\ref{tb:multi_seed} reports the standard deviations across all three runs.
CoCA's benchmark-level standard deviations range from $0.22\%$ to $0.48\%$, comparable to RaR's range of $0.18\%$ to $0.53\%$.
For the seven-benchmark average, CoCA exhibits a lower standard deviation than RaR ($0.17\%$ versus $0.24\%$).
Moreover, its $3.1\%$ average lead over the strongest baseline in Table~\ref{tb:main} far exceeds the observed variation, indicating that CoCA's SOTA average performance remains robust across the evaluated seeds.

\section{Pseudocode of CoCA}
\label{app:pseudocode}

We describe the CoCA pipeline in Algorithm~\ref{alg:coca}.

\begin{algorithm}[!ht]
\caption{CoCA: Criterion-Aware Credit Assignment}
\label{alg:coca}
\small
\linespread{1.1}\selectfont
\begin{algorithmic}[1]
\Require Dataset $\mathcal D$, policy $\pi_\theta$, judge $J$, group size $G$,
allocation budget $k$, concentration $\beta$.
\For{each training iteration}
    \State Freeze the rollout policy $\pi_{\theta_{\mathrm{old}}}\gets\pi_\theta$
    \State Sample a batch of prompt--rubric pairs from $\mathcal D$
    \For{each $(x,\mathcal R)$ in the batch}
        \State Sample $\{y_i\}_{i=1}^{G}\sim\pi_{\theta_{\mathrm{old}}}(\cdot\mid x)$
        \State Obtain criterion outcomes $\{z_{ij}\}$ using $J$
        \State $R_i\gets \sum_j s_jz_{ij}/\sum_{j:s_j>0} s_j$ for all $i$
        \State $\bar z_j\gets G^{-1}\sum_i z_{ij}$ for all $j$
        \State $A_{ij}\gets s_j(z_{ij}-\bar z_j)/[\sum_{\ell:s_\ell>0}s_\ell]$
        for all $i,j$
        \For{$i=1,\ldots,G$}
            \State $T_i\gets|y_i|$;
            $\widehat A_{i,t}\gets\sum_j A_{ij}$ for all $t$
            \State $\mathcal S_i\gets
            \operatorname{TopK}_{j:A_{ij}\neq0}(|A_{ij}|,k)$
            \State $p_{i,t}\gets
            \pi_{\theta_{\mathrm{old}}}(\cdot\mid h_{i,t})$ for all $t$
            \For{each $j\in\mathcal S_i$}
                \State $q_{ij,t}\gets
                \pi_{\theta_{\mathrm{old}}}(\cdot\mid h_{i,t}(c_j))$
                for all $t$
                \State $d_{ij,t}\gets
                \operatorname{sg}[D_{\mathrm{JS}}(p_{i,t}\|q_{ij,t})]$
                for all $t$
                \State $\boldsymbol\omega_{ij}\gets
                T_i\operatorname{softmax}(\beta\mathbf d_{ij})$
                \State $\widehat A_{i,t}\gets\widehat A_{i,t}
                +A_{ij}([\boldsymbol\omega_{ij}]_t-1)$ for all $t$
            \EndFor
        \EndFor
    \EndFor
    \State Update $\theta$ on the batch using Eq.~(\ref{eq:grpo}),
    replacing $A_i$ with fixed $\widehat A_{i,t}$
\EndFor
\State \Return $\pi_\theta$
\end{algorithmic}
\end{algorithm}

\end{document}

%% file: math_commands.tex
\usepackage{amsmath,amsfonts,bm}

\def\eqref#1{equation~\ref{#1}}

\def\1{\bm{1}}

\DeclareMathAlphabet{\mathsfit}{\encodingdefault}{\sfdefault}{m}{sl}
\SetMathAlphabet{\mathsfit}{bold}{\encodingdefault}{\sfdefault}{bx}{n}

